\documentclass[times]{elsarticle}

\journal{journal}

\usepackage[linesnumbered]{algorithm2e}
\usepackage{amsthm,tikz} 
\usepackage{fullpage}
\usepackage{amsmath}
\usepackage{amssymb,txfonts}
\usepackage{graphicx}
\usepackage{pifont}
\usepackage{fancyvrb}
\usepackage{caption}
\usepackage{subcaption}
\usepackage{stmaryrd}

\newtheorem{THEOREM}{Theorem}

\newenvironment{theorem}{\begin{THEOREM}  }%
                        {\end{THEOREM}}

\newtheorem{LEMMA}[THEOREM]{Lemma}
                      {\end{LEMMA}}
\newtheorem{COROLLARY}[THEOREM]{Corollary}
                          {\end{COROLLARY}}
\newtheorem{PROPOSITION}[THEOREM]{Proposition}
\newenvironment{proposition}{\begin{PROPOSITION}  }%
                            {\end{PROPOSITION}}
\newtheorem{DEFINITION}[THEOREM]{Definition}
\newenvironment{definition}{\begin{DEFINITION}  \rm}%
                            {\end{DEFINITION}}
\newtheorem{CLAIM}[THEOREM]{Claim}
                            {\end{CLAIM}}
\newtheorem{EXAMPLE}[THEOREM]{Example}
\newenvironment{example}{\begin{EXAMPLE} \rm}%
                            {\end{EXAMPLE}}
\newtheorem{REMARK}[THEOREM]{Remark}
                            {\end{REMARK}}
							\newtheorem{NOTATION}[THEOREM]{Notation}
							                            {\end{NOTATION}}
\DeclareMathAlphabet{\mathitbf}{OML}{cmm}{b}{it}

\newcommand{\sub}{_}
\def\su{^}
\newcommand{\real}{{\mathbb{R}}}

\newcommand{\wmc}{\textsc{WMC}}

\newcommand{\gt}{>}

\newcommand{\C}{{\cal C}}

\newcommand{\F}{{\cal F}}

\newcommand{\kb}{\Delta}

\newcommand{\U}{{\cal U}}

\newcommand{\set}[1]{\left\{ #1 \right\}}

\newcommand{\pr}{\Pr}

\newcommand{\high}{\sub{h}}
\newcommand{\low}{\sub{l}}
\newcommand{\kbh}{\kb\high}
\newcommand{\kbl}{\kb\low}
\newcommand{\kbll}{\U\sub l}%{\kb\sub 1}
\newcommand{\kbhh}{\U\sub h}%{\kb\sub 2}
\newcommand{\kblll}{\C\sub l}%{\kb\sub 1}
\newcommand{\kbhhh}{\C\sub h}%{\kb\sub 2}
\newcommand{\wh}{w\high}
\newcommand{\wl}{w\low} 
\newcommand{\mh}{M\high}
\newcommand{\ml}{M\low}
\newcommand{\iso}{\sim\sub m}
\newcommand{\vecx}{\vec{x}}

\newcommand{\lang}{\textit{Lang}}
\newcommand{\lits}{\textit{Lits}}
\newcommand{\mods}{\textit{Models}}
\newcommand{\atom}{p}

\begin{document}
	
\begin{frontmatter}
		
		%\\{Abstracting Probabilistic Models}
		
	% \title{Logical Considerations on Abstracting Probabilistic Models}
	
	%\title{Logical Perspectives on Abstracting Probabilistic Models}
	 
	 %\title{Abstractions via Weighted Model Counting}
	 
	 %\title{Abstracting Probabilistic Logical Models}
%
% 		

%\title{Foundations for  Abstracting Probabilistic Logical Models}
%\title{Abstracting  Probabilistic Logical Models}
%\title{Logical and Probabilistic Abstractions of Probabilistic Logical Models}
%\title{Abstraction in Probabilistic  Knowledge Bases}
%\title{Abstraction in   Probabilistic Models}
%\title{Abstracting  Probabilistic Representations}
%\title{Abstraction in  Expressive  Probabilistic Models}
%
%
%\begin{document}

\title{Abstracting   Probabilistic Models:  A Logical Perspective\thanks{The author was supported by a Royal Society University Research Fellowship.}}

		\cortext[cor]{Corresponding author. The author was supported by a Royal Society University Research Fellowship.}

		 \address[edinburgh]{School of Informatics, University of Edinburgh,
		Edinburgh, UK.}
		 \address[turing]{Alan Turing Institute,
		London, UK.}

		\author[edinburgh,turing]{Vaishak Belle\corref{cor}}
		\ead{vaishak@ed.ac.uk}

		% %
	% 	%
	% 	\author{Vaishak Belle \\
	% 	University of Edinburgh \& Alan Turing Institute \\
	% 	\email{vaishak@ed.ac.uk}
	% 	}
	%
		% % \ead{vaishak@ed.ac.uk}
		%
%
%
%
%

% \twocolumn[
%
% \aistatstitle{Abstracting Probabilistic Relational Models}
%
% \aistatsauthor{  }
%
% \aistatsaddress{ } ]

%\maketitle
%

%\maketitle 

\begin{abstract} Abstraction is a powerful idea  widely used in science, to model, reason and explain the behavior of systems in a more tractable search space, by omitting irrelevant details. While notions of abstraction have matured for deterministic systems, the case for abstracting probabilistic models is not yet fully  understood. 

In this paper, we provide a semantical framework for  analyzing such abstractions  from first  principles. We develop the framework in a general way, allowing for expressive languages, including logic-based ones that admit relational and hierarchical constructs with stochastic primitives.   We motivate a definition of consistency between a high-level model and its low-level counterpart, but also treat the case when the high-level model is missing critical information present in the low-level model.  We go on to prove prove properties of abstractions, both at the level of the parameter as well as the structure of the models. 
We conclude with some observations about how abstractions can be derived automatically.

\end{abstract}

% \begin{keyword}
% 	Knowledge representation \sep
% 	Weighted model counting \sep
% 	Abstraction \sep
% 	Reasoning about uncertainty
% \end{keyword}

\end{frontmatter}

\section{Introduction}

\emph{Abstraction} is a powerful idea widely used in science to explain phenomena at the required granularity.  Think of explaining a heart disease  in terms of  its anatomical components versus its molecular composition. Think of understanding the political dynamics of elections by studying micro level phenomena (say, voter grievances  in  counties) versus  macro level events (e.g., television  advertisements, gerrymandering). In particular, in computer science, it is often understood as the process of mapping one representation onto a  simpler representation by suppressing irrelevant information.  
The motivation is three-fold: 

\begin{enumerate}

\item[(a)] When representing complex pieces of knowledge, abstraction can provide a way to structure that knowledge, hierarchically or otherwise, so as to yield descriptive clarity and modularity. 

\item[(b)] Reasoning over large  graphs, programs, and other  structures is almost always computationally challenging, and so abstracting the problem domain to a  smaller search space is attractive. Even in the case of tractable representations, such as  arithmetic circuits {\cite{a-knowledge-compilation-map}},  reasoning is  polynomial in the circuit size, so clearly a smaller circuit  is more effective.

\item[(c)] Lastly,  and perhaps most significantly, 
 abstraction features pervasively in commonsense reasoning, and there is much discussion in the fields of cognitive science and philosophy  on the role of abstractions for explanations \cite{jorland1994idealization,doi:10.1111/tops.12278}; for example,  \cite{garfinkel1981forms} argues that concrete explanations containing too much detail are sensitive to  perturbations and are impractical for understanding physical phenomena. Thus, abstractions will likely be critical  for {\it explainable AI} \cite{explainable-artificial-intelligence-xai}, and indeed, much of that literature focuses on extracting high-level symbolic and/or programmatic representations from low-level data (e.g., \cite{DBLP:journals/corr/PenkovR17,sreedharan2018hierarchical}).

\end{enumerate} 

Formal perspectives on abstraction  have matured considerably over the years \cite{giunchiglia1992theory,milner1989communication,banihashemi2017abstraction}. In particular, the work of \cite{banihashemi2017abstraction} is  noteworthy as it identifies how notions of soundness and completeness relate to the model-theoretic properties of a high-level abstraction and the corresponding low-level theory. 
However, the formal analysis of abstraction has largely focused on categorical (deterministic and non-probabilistic) domains; that is,  both the high-level and the low-level representations are assumed to be categorical  assertions. 
In that regard, existing frameworks are not immediately applicable to the fields of probabilistic modeling and statistical machine learning.  Indeed, we do not yet have a full  understanding  of which aspects of one probabilistic model, representing some low-level phenomena, can be omitted when building a less granular (possibly non-probabilistic) model standing for a high-level understanding of the domain.

 In this paper, we provide a semantical framework for  analyzing such abstractions  from first  principles. We develop the framework in a general way, allowing for expressive languages, including logic-based ones that admit relational and hierarchical constructs with stochastic primitives \cite{probabilistic-models-for-relational-data,an-introduction-to-statistical-relational-learning}. Representative  examples of such languages include probabilistic databases  and statistical knowledge bases, which have received considerable attention both in the academic and  industry circles \cite{probabilistic-databases,markov-logic-networks,probase:-a-probabilistic-taxonomy-for-text,knowledge-vault:-a-web-scale-approach,deepdive:-web-scale-knowledge-base-construction,toward-an-architecture-for-never-ending-language}. 
 %Like that literature, we primarily focus on finite-domain (essentially propositional) models 

% In this work,  we develop a foundational framework for abstraction in {\it probabilistic relational models} from first principles. Probabilistic relational models (PRMs)  generalize standard (propositional) probabilistic models in borrowing syntactic constructs from first-order logic \cite{probabilistic-models-for-relational-data,an-introduction-to-statistical-relational-learning}. Thus, our results are applicable to a very general class of   probabilistic models that are additionally able
% to reason about relational data and hierarchical constraints, as applicable to formalisms such as probabilistic databases  \cite{probabilistic-databases,markov-logic-networks} and statistical knowledge bases \cite{deepdive:-web-scale-knowledge-base-construction,toward-an-architecture-for-never-ending-language}.

%, which are often built by means of natural language processing techniques applied to Web and unstructured textual data so as to yield weighted logical atoms.   

In this work, we motivate a definition of  consistency between a high-level (probabilistic or logical) model and its low-level (probabilistic)   counterpart, but also treat the case when the high-level model is missing critical information present in the low-level model.
%Our starting point is the work of \cite{banihashemi2017abstraction}, which enables a simple yet precise way to logically characterize the differences between a high-level theory and the low-level one, via a  well-understood formulation of isomorphism.   
 % We approach this construction by appealing to the notion of weighted model counting , which can be seen both as a semantical as well as a computational framework underlying many popular probabilistic logical representations.
We  go on to prove properties of abstractions, both at the level of the parameter as well as the structure of the models. Put differently, we first motivate a definition of abstraction purely at the level of the model theory, which then provides the basis for analyzing the properties of ``unweighted abstractions.'' (That is, probabilities are simply ignored in that construction.) We  use that  analysis to investigate how ``weighted abstractions'' can be defined. We then  study how to incorporate low-level evidence and reason about it in the high-level representation. 
%We approach these results by appealing to the notion of weighted model counting \cite{on-probabilistic-inference-by-weighted-model}, which serves as an assembly language for many popular probabilistic logical representations. 
We conclude with some observations about how abstractions can be derived automatically.

With the development of this framework, we hope to provide a formal basis for  developing probabilistic abstractions in service of increased   modularity,  tractability and interpretability.

\section{Desiderata} %
\label{sec:desiderata}

Before developing a framework for abstraction, let us briefly reflect on what is desired of such a framework.  To a first approximation, a formal theory of abstraction can be approached  in three stages:
 
  \begin{enumerate}
 	\item How should abstraction be defined between  a high-level representation \( \kbh \)  and a low-level one \( \kbl \)?

%	, and under what conditions can we say that  a high-level representation \( \kbh \)  is an abstraction of a low-level representation \( \kbl \)?
	% {What does it mean to say that a high-level representation \( \kbh \)  is an abstraction of a low-level representation \( \kbl \)?}
% 	That is, how do we motivate  a formal theory for defining abstraction?
	\item Given \( \kbh \) and \( \kbl, \)
	how do we prove that \( \kbh \) is an abstraction of \( \kbl \)? 
	\item Given \( \kbl \) and a target high-level vocabulary, how do we find \( \kbh? \)

 \end{enumerate}
At the outset, in this work,  we are  concerned with (1) and (2), but we will also consider a preliminary investigation of (3).

 % (Needless to say, both of these are necessary before  an algorithmic regime for (3) can be engineered.)

In essence, abstractions are  about omitting irrelevant details, while providing a less granular language to capture and reason about the underlying probabilistic components. To motivate that using an example, consider a \emph \emph{probabilistic relational model}  (PRM)  on entity-relationships for a university database \( \U \) (adapted from \cite{probabilistic-models-for-relational-data}). The model instantiates constraints for a (parameterised) Bayesian network:

\tikz {
\node (a) at (0,2) {Difficulty}; 
\node (b) at (2,2) {Grades}; 
\node (c) at (4,2) {IQ};
\draw (a) edge[->] (b) (c) edge[->] (b);
}

\newcommand{\dis}[1]{{\{#1\}}}

\newcommand{\Pf}{{\it friends}}
\newcommand{\Pd}{{\it diff}}
\newcommand{\Piq}{{\it iq}}
\newcommand{\Pa}{{\it advises}}
\newcommand{\Pt}{{\it takes}}
\newcommand{\Ptt}{{\it teaches}}
\newcommand{\Pg}{{\it grades}}
\newcommand{\Ce}{{E}}
\newcommand{\Cm}{{M}}
\newcommand{\Ch}{{H}}
\newcommand{\Cp}{{P}}
\newcommand{\Cf}{{F}}
\newcommand{\Cl}{{L}}
\newcommand{\Cn}{{N}}
\newcommand{\Cb}{{B}}
\newcommand{\Co}{{O}}
\newcommand{\Cg}{{G}}
 as follows, referred to as the low-level theory \( \kbll \) in the sequel: \begin{itemize}
	\item[\textbf{0.7}]  \( \Pd(x,\Ce) \)
	\item[\textbf{0.1}]  \( \Pd(x,\Cm) \)
	
	\item[\textbf{0.2}] \( \Pd(x,\Ch) \)
	\item[\textbf{0.25}] \( \Piq(x,\Cl) \land \Pd(y,\Ce) \land \Pt(x,y) \supset \Pg(x,y,u)  \) for \( u \in \set{7,8,9,10} \)
	\item[\textbf{0.25}] \( \Piq(x,\Cl) \land \neg \Pd(y,\Ce) \land \Pt(x,y) \supset \Pg(x,y,u) \) for \( u \in \set{5,6,7,8} \) 
\end{itemize}
where the constants \( \Ce,\Cm, \Ch, \Cl \) stand for \emph{easy, medium, hard, low} respectively. (A precise encoding will be presented in a subsequent section.) 

  The first constraint says that for any given  course, say \( B, \)  the probability that its difficulty level is easy is 0.7. The fourth constraint says that for any low IQ student taking an easy course, the probability that his grade is 7 is \( 0.25 \), and likewise, the probability that his grade is 8 is \( 0.25 \), and so on. 
  More generally, this theory  says that courses come in three levels of difficulty, and when a low IQ student takes an easy course, his grades can be modeled as a uniform distribution on \( \set{7,8,9,10} \), and when he does not take an easy course, it is a uniform distribution on \( \set{5,6,7,8} \).

A simple yet powerful type of abstraction to apply here is to abstract away the domain. 
% \emph{domain abstraction}.   
Assuming the above sentences are the only ones of interest to us, we can lump the constants \( \set{\Cm, \Ch} \) as \( \Cn \), standing for \emph{not easy}, and lump the mentioned grade values together as \( \set{5,6}, \set{7,8}, \set{9,10} \) and denote them as \( \Cb, \Co, \Cg \), standing for \emph{bad, ok} and \emph{good} respectively. Then, we would obtain the following model, referred to as the high-level theory \( \kbhh \) in the sequel:\footnote{Although the abstraction uses the same predicates as \( \kbll, \) note that some of these are essentially new predicates, with different domains. For example, in \( \kbll, \) the difficulty ranges over \( \set{\Ce,\Cm,\Ch} \) whereas in \( \kbhh, \) it ranges over \( \set{\Ce,\Cn}. \) The context will make clear whether the predicates and constants are from \( \kbll \) or from \( \kbhh, \) and so we do not distinguish symbols from \( \kbhh \) by means of superscripts and such. }  \begin{itemize}
	\item[{\bf .7}]   \( \Pd(x,\Ce) \)
	\item[{\bf .3}]  \( \Pd(x,\Cn) \)
	\item[\textbf{.5}]  \( \Piq(x,\Cl) \land \Pd(y,\Ce) \land \Pt(x,y) \supset \Pg(x,y,u)  \) for \( u \in \set{\Co,\Cg} \)
	\item[\textbf{.5}] \( \Piq(x,\Cl) \land  \Pd(y,\Cn) \land \Pt(x,y) \supset \Pg(x,y,u) \) for \( u \in \set{\Cb,\Co} \) 
\end{itemize}

On closer inspection, the reader may observe that \( \kbhh \) is, in fact, a very faithful abstraction of \( \kbll  \), in terms of accurately grouping together probabilistic events. Indeed, we will formally show that  the two models  agree on  a large class of probabilistic queries.   The benefit, of course, is that \( \kbhh \) is defined over a smaller set of random variables.

However, such a faithful alignment may not always be needed, or even feasible. Consider a case where we abstract by  grouping definitions and complex formulas using new predicates. Suppose we had a course listing database \( \C. \) Let \( \C\sub l \) be a low-level theory: \begin{itemize}
	\item[{\bf .9}] \( {\it CS(x)} \supset \Pd(x,\Ch)  \) 
	
	\item[{\bf .8}] \( {\it Physics(x)} \supset \Pd(x,\Ce) \)
	\item[{\bf 1}] \( ({\it AI}(x) \supset {\it CS}(x)) \land  ({\it Astronomy}(x) \supset {\it Physics}(x)) \) 
\end{itemize}
We may want to define a high-level theory \( \C\sub h \) that simply uses \( {\it Science}(x) \) in place of 	\( {\it CS}(x) \) and \( {\it Physics}(x) \).
But then the weight on  rules such as \( {\it Science}(x) \supset \Pd(x,\Ch) \) or \( {\it Science}(x) \supset \Pd(x,\Ce) \)
may not be immediate  to derive, in general.  %
Predicate abstraction can also be used as a strategy to check for probabilistically significant events. For example, an administrator may  only be interested in ensuring that all low IQ students enroll in an easy course: \( \emph{alert} \doteq \neg~ [\forall x, \exists y ~(\Piq(x,L) \supset (\Pt(x,y) \land \Pd(y,\Ce)))]  \)
%
% \begin{itemize}
% 	\item[] \( \emph{alert} \doteq \neg~ [\forall x, \exists y ~(\Piq(x,L) \supset (\Pt(x,y) \land \Pd(y,\Ce)))]  \)
% 	\end{itemize}
and specifically, whether that atom ever obtains a non-zero probability. 
Indeed, the literature on verification and security often approach  the reasoning of complex systems by distinguishing \emph{bad} states (e.g., invalid paths, safety conditions) \cite{sharma2013verification}, and  correspondingly, checking whether such states are probable or improbable. Naturally, by means of a relational language, such definitions can be arbitrarily complex and hierarchical, and different from classical works on categorical abstraction,  predicates at every level can denote  stochastic primitives.

In that spirit, we show that abstraction can be understood  both from the viewpoint of the parameters (i.e., weights and/or probabilities) and 
structure (i.e., the logical sentences).  While we do discuss the case of aligning probabilities exactly  between the high-level and low-level models, we also consider the most immediate  case of parameter abstraction where one obtains an alignment between the probable and improbable events. When it comes to abstracting structure, we show that one wants to ensure that the high-level model is consistent, and perhaps additionally that it is not missing critical information present at the low-level model. This  then motivates a  definition of \textit{soundness} and \textit{completeness}. 

Our starting point was the work of~\cite{banihashemi2017abstraction}, which introduces a simple  way to logically characterize the differences between a high-level theory and the low-level one, via the well-understood notion of isomorphisms. We show how that account can be extended to reason about probabilities by appealing to the formulation of \emph{weighted model counting}  \cite{on-probabilistic-inference-by-weighted-model}, which serves as an assembly language for many popular  PRMs. The resulting treatment can be seen to share much of the simplicity of \cite{banihashemi2017abstraction}, thereby providing an amenable framework  for understanding probabilistic and logical abstractions of PRMs. 

We reiterate that our focus here is primarily about the semantic constraints for analyzing abstractions. Thus, at the outset, we assume that we are given a  \emph{high-level theory}, capturing the more abstract probabilistic model, and a \emph{low-level theory}, understood as the underlying probabilistic model that is to be abstracted. Nonetheless, we conclude our technical treatment by  discussing some ideas for deriving abstractions automatically.

\section{Preliminaries} 

Our technical development will discuss  the semantical constraints between different representations, defined in terms of  a mapping between probabilistic events. For the purpose of our results, it will be useful to think in terms of these representations being knowledge bases (i.e., sentences in some logical language), over which one defines a   measurable space \( (S, \F_S) \)  \cite{halpern2004representation}.  In  particular, for any given knowledge base \( \Delta \), we imagine \( S \) to be some subset of  the set of interpretations of \( \Delta \). Moreover, when analyzing how precisely two representations agree, we will be considering the probabilities of queries that additionally use  logical connectives such as conjunction and negation, and so we will require that measures  be well-defined over such connectives.  \smallskip

%For the sake of concreteness, we define a logical language below over a finite vocabulary, although more general treatments are possible \cite{halpern2004representation}. 

% \footnote{In principle, it suffices to think of any arbitrary measurable space $(X, \F_X ),$ where a measurable space consists of a set $X$ and an algebra  of subsets of $X$ with the constraints that: for every \( S \in  $\F_X$  \), \( \Pr(S) \in \real_{[0,1]} \), that }

% We will also want to speak of how precisely these representations agree; for example, we could expect them to agree in terms of the probabilities of queries that may additionally make use of logical connections such as conjunction and negation. In particular, consider
%
% if we imagine a representation \( \Delta, \)
%
% We then expect a
%
%
% Thus,
%
% is at the level of states
%
%
% There are two ways probabilistic inference are captured in the literature . In the first, In one, we are given a prior distribution over some probability space; our “knowledge” then typically consists of events in that space, which can be used to condition that distribution and obtain a posterior. In the other, which is the focus of our work, a probabilistic inference procedure takes as input a probabilistic knowledge base and returns a probabilistic conclusion.

%\subsection{Logical Language}

% Concretely, imagine a logical language \( \L \) with atoms,

%For the sake of concreteness, we 

Concretely, define a relational language \( \lang \) with  predicate symbols  of every arity \[ \set{P_1(x), \ldots, P_2(x,y), \ldots, P_3(x,y,z), \ldots}, \] variables \( \set{x,y, z, \ldots} \), connectives \( {\lor, \neg, \land, \forall} \) and a  set of constants \( \set{c\sub 1, c\sub 2, \ldots} \), 
serving as the {\it domain of discourse} for quantification. 
To facilitate comparisons  between vocabularies, we assume that for each high/low-level theory  the relations and domain are finite subsets of this  fixed infinite vocabulary. For simplicity, we restrict our attention to probability spaces over finitely many random variables, as would be instantiated from our assumption. This would be applicable to most  statistical relational languages, such as probabilistic databases,   Markov logic networks and knowledge graphs  \cite{probabilistic-databases,markov-logic-networks,knowledge-vault:-a-web-scale-approach}. Although from a logical viewpoint, we could simply have used a propositional one, we will introduce  a relational language, as is usual in the literature  \cite{statistical-relational-ai:-logic-probability}.\footnote{If there are infinitely many random variables instantiated from the first-order language, we may consider countably additive probability measures \cite{concerning-measures-in-first-order,reasoning-about-uncertainty} or other  syntactic  conditions, as in, for example, \cite{graphical-markov-models-for-infinitely,reasoning-about-infinite-random,markov-logic-in-infinite-domains,approximate-inference-for-infinite-contingent,open-universe-weighted-model-counting}.}  

Standard abbreviations apply for connectives: we write \( \alpha\supset \beta \) (material implication) to mean \( \neg\alpha\lor \beta \), \( \alpha\equiv \beta \) (equivalence) to mean \( (\alpha\supset \beta) \land (\beta\supset\alpha) \), 
and \( \exists x\alpha \) (existential quantification) to mean \( \neg\forall x \neg \alpha \). In particular, when the domain is fixed to a finite set $D$, we write $\forall x~ \alpha(x)$ to mean $\bigwedge_{c\in D} \alpha(c)$.  Moreover, $\alpha \land \beta$ is equivalent to $\neg(\neg \alpha \lor \neg \beta)$, so in proofs, we  only consider the connectives \( \set{\land,\neg}. \) 

The set of ground atoms is defined as: \[ 
 \set{P(c\sub 1, \ldots, c\sub k) \mid \textrm{$P$ is a relation, $c\sub i \in D$}}. 
\] The set of ground literals is obtained from the set of atoms, and their negations. Henceforth, when we write atoms and literals, we will implicitly mean ground ones. We often use \( p \) and \( q \) to denote atoms,  and \( l \) and \( d \) to denote literals.

A model \( M \) is a \( \set{0,1} \) assignment to the set of atoms. Using \( \models \) to denote satisfaction, the semantics for a formula \( \phi \) is defined inductively: $M \models \atom$ for atom $\atom$ iff $M[\atom] = 1$; $M \models \neg \phi$ iff $M \models \phi$ does not hold (also written \( M\not\models \phi \)); $M \models \phi\lor \psi$ iff $M \models\phi$ or $M\models \psi$; and \( M\models \phi\land \psi \) iff \( M\models \phi \) and \( M\models \psi \). We write $l\in M$ to mean that \( M \models l \) for literal \( l. \)

We say a formula \( \phi \) is \emph{satisfiable} iff there is a model \( M \) such that \( M\models\phi. \) 
 We write  \( \kb \models \phi \) to mean that in every model $M$ such that $M\models \kb$, it is also the case that $M \models\phi$. In particular, we say that \( \phi \) is \emph{valid}, written \( \models \phi, \) iff for every model \( M \), \( M\models \phi \).

To prepare for our technical discussion, we discuss some notational conventions. Given a formula \( \kb \), 
we write \( \lang(\kb) \) to mean the the logical sub-language implicit in \( \kb \): that is, the set of well-formed formulas constructed from  relations \( \set{P_1(x), \ldots} \) and constants \( D \) mentioned in \( \kb \). We can then write $\alpha \in \lang(\kb)$ to mean such as well-formed formula. Analogously, we write \( \lits(\kb) \) to mean the set of literals obtained from \( \lang(\kb) \). For example, if \( \kb = P(c) \lor Q(c,a) \), then 
 $\neg P(a) \in \lang(\kb), Q(a,a) \in \lang(\kb)$,  \( P(a) \in \lits(\kb), \neg Q(a,c) \in \lits(\kb) \), and so on. 
  We often abuse notation and write \( \vec c\in D \) to mean that each of the constants mentioned in \( \vec c \) is taken from \( D. \) Finally, given a \( \kb, \) when we write \( M\models \kb \), it is implicit here that we take \( M \) to be a model for the language \( \lang(\kb) \); that is, it  is  a \( \set{0,1} \) assignment to the set of atoms in \( \lang(\kb) \). We can make this explicit by  writing \( M \in \mods(\lang(\kb)) \), or simply \( M\in \mods(\kb) \) for short.\footnote{The reason we go to some length to discuss our notational conventions is this: when we work with a fixed language, the set of relations, literals, and models to consider is immediate. That will no longer be true when we are thinking of different logical languages for  high-level and low-level theories, in which case  our notation will  provide context.} \smallskip

As hinted above, we will now assume that for any \( \Delta \), we are given a measurable space \( (S, \F_S) \), where \( S\subseteq \mods(\Delta) \) \cite{halpern2004representation}. Since \( \mods(\kb) \) is finite,  let \( S = \mods(\Delta) \)  for simplicity. For this measurable space, we further assume that for every \( \alpha\in \lang(\kb) \), \( \Pr(\alpha)  \) is a numeric term;  that is, every well-defined formula is accorded a probability. We further interpret a conditional probability expression as \[ \Pr(\alpha\mid \beta ) =  \frac{\Pr(\alpha\land\beta)}{\Pr(\beta)} \] 

%(See \cite{halpern2004representation} on how to define such a formal framework without committing to any specific language.)

This interpretation places very little restrictions on the computational machinery that one may use \cite{halpern2004representation}. 
For the sake of concreteness, {\it weighted model counting} (WMC) {\cite{solving-sat-and-bayesian-inference-with}}, for example, is a reasonable fit. We remark that nothing in our technical treatment hinges on using WMC, and we only use the framework to illustrate examples and the encoding for the university PRM. In some cases, we state useful properties of WMC, but these would hold in virtually all statistical relational languages and probabilistic logics \cite{halpern2004representation,statistical-relational-ai:-logic-probability,reasoning-about-knowledge-and-probability}. 

WMC is defined over the models of a propositional formula, and serves as an assembly language for a number of heterogeneous representations, including factor graphs, Bayesian networks, probabilistic databases and probabilistic programs \cite{solving-sat-and-bayesian-inference-with,probabilistic-databases,inference-in-probabilistic-logic-programs}. 
WMC enjoys a number of interesting properties that makes it particularly well-suited for our endeavor. First, it separates the symbolic representation (i.e., a logical encoding of the probabilistic model) from a weight function denoting the probabilities of variables, which allows us to investigate abstractions both at the level of structures and at the level of parameters. Second, WMC provides a semantic as well as a computational view for probabilistic reasoning. Semantically, the models of propositional formulas map to  \emph{states} in probability spaces (i.e., assignments of values to random variables). Computationally, we are able to reuse SAT  technology for building exact and approximate solvers \cite{model-counting}, while still  leveraging context-specific independences {\cite{context-specific-independence}}. 
In particular, recent approaches for WMC \cite{on-probabilistic-inference-by-weighted-model}  such as knowledge compilation \cite{a-knowledge-compilation-map} provide effective ways for enumerating and testing properties on propositional interpretations. 

Essentially,  WMC extends \textit{model counting}, which is the task of counting the models of a propositional formula \cite{model-counting}. In WMC, weights are  additionally accorded to  literals, and we are interested in summing the weights of the models, which is then defined in terms of the product of the literal weights. Standard probabilistic inference, WMC and model counting are, in fact, closely related problems, with polynomial time reductions to each other, with their decision versions being \#P-hard \cite{solving-sat-and-bayesian-inference-with,the-complexity-of-enumeration-and-reliability-problems}. Formally,\footnote{We define WMC at the level of the ground theory. In the literature, however, a special case of WMC is sometimes considered for relational languages, where the weight function maps predicates directly to numbers (e.g.,   \cite{lifted-inference-and-learning-in-statistical}). The intuitive idea is to treat this weight function as a template for all instances of the corresponding predicate, which, on the one hand, simplifies the specification of the weight function, and on the other, admits effective inference. We do not discuss such ideas here as it is orthogonal to the main thrust of this work (cf. penultimate section).} 

\begin{definition} Suppose \( \kb \) is a ground first-order sentence. Suppose $w$ is a function that maps the elements of \( \lits(\kb) \) to \( \real\su {[0,\infty)} \). Then the WMC of \( \kb \) is defined as: \[
	\wmc(\kb,w) = \sum\sub{M\models \kb} \prod\sub{l\in M} w(l) 
\]
Given a formula \( \phi \in \lang(\kb) \),  we can query \( \phi \) wrt evidence \( e \) for theory \( (\kb,w) \) using: \begin{equation}\label{eq:cond}\begin{array}{r}\tag{$\ddag$}
	\displaystyle \pr(\phi\mid e, \kb, w) = \frac{\wmc(\phi\land e\land \kb, w)}{\wmc(e\land \kb, w)} \\ =  \displaystyle \frac{\pr(\phi \land e,\kb,w)}{\pr( e,\kb,w)}  
\end{array}	 
	%\pr(\phi\mid e, \kb, w) = \frac{\wmc(\phi\land e\land \kb, w)}{\wmc(e\land \kb, w)}
\end{equation}
\end{definition}
When \( e = true, \) we simply write \( \Pr(\phi,\kb,w). \)
We remark for \( \Pr(\phi,\kb,w) \) to be well-defined, which is assumed, \( \wmc(\kb,w) \neq 0. \) (Thus, it is assumed that  \( \kb \) is satisfiable, and that $w$ does not map all the corresponding literals to 0.) If the context is clear, we often refer to \( \kb \) as the \emph{theory}, and to \( \phi \) as the \emph{query} or \emph{event}.

We immediately observe the following property from the definition of \( \wmc \).

%\footnote{Properties based on entailments and satisfiability that we report here for WMC also hold  in most  probabilistic logical frameworks \cite{halpern2004representation}. }

% ({Proofs are provided in the supplementary material.})

\begin{theorem}\label{prop:wmc entailment}  If \( \kb\models \phi, \) then \( \pr(\phi,\kb,w) = 1. \) If \( \kb\land\phi \) is not satisfiable, then \( \pr(\phi,\kb,w) = 0. \)
	
\end{theorem}

\begin{proof}
For the first property, every \( M \) such that \( M\models\kb \), \( M\models \phi \) also, and so \( \wmc(\phi\land\kb,w) = \wmc(\kb,w) \). For the second, \( \wmc(\phi\land \kb,w) = 0. \)   	
\end{proof}

\begin{example} We illustrate a WMC  encoding for \( \kbll \) based on the university PRM;  the encoding for others considered in this work are analogous. First, note that in atoms such as \( \Pd(x,y) \), the logical variable \( y \) captures the possible values of a random variable. Thus, they are to behave like logical functions. Formally, let \( \kbll \) be the union of the following, the free variables being implicitly universally quantified from the outside: 
	\begin{itemize}
	\item $\Pd(y,E) \lor \Pd(y,M) \lor \Pd(y,H)$
	\item \( f\sub 1 (x,y, u) \equiv [\Piq(x,\Cl) \land \Pd(y,\Ce) \land \Pt(x,y) \supset \Pg(x,y,u)  ] \) for \( u \in \set{7,8,9,10} \)
	\item \( f\sub 2 (x,y,u) \equiv [\Piq(x,\Cl) \land \neg \Pd(y,\Ce) \land \Pt(x,y) \supset \Pg(x,y,u)  ] \) for \( u \in \set{5,6,7,8} \) 
\end{itemize} The reason we need to introduce auxiliary predicates \( f\sub 1 \) and \( f\sub 2 \) is because WMC only allows weights on (ground) literals.

We also need the following hard constraints for capturing the logical functions: \begin{itemize}
	\item[] \( \exists u (\Pd(y,u)), \Pd(y,u) \land \Pd(y,v) \supset u = v  \) 
	\item[] \( \exists u (\Pg(x,y,u)), \Pg(x,y,u) \land \Pg(x,y,v) \supset u = v \)
\end{itemize}
 Suppose the domain of quantification for the students is only \( \set{A} \) and for courses is only \( \set{B} \). We then obtain atoms such as: \begin{itemize}
	\item[] \( \Pd(B,E), \Pd(B,M), \Pd(B,H), \Piq(A,\Cl),  \) \( \Pd(B,\Ce), \Pt(A,B), \Pg(A,B,7), \ldots \)
\end{itemize}
with a weight function \( \wl \)  for positive atoms derived from the parametric specification in an obvious fashion: \begin{itemize}
	\item[] \( w\sub l(\Pd(B,E)) = .7, \ldots, \wl(f\sub 1(A,B,7)) = .25, \ldots  \) 
\end{itemize}
We let the weight of a negated atom \( \wl(\neg a) \) to be \( 1 - \wl(a). \) Moreover, the ground instances \( f\sub 1 \) and \( f\sub 2 \) obtain the weights discussed in the parameterized version. The weights of all atoms not mentioning predicates \( \Pd, f\sub 1, f\sub 2 \) is taken to be 1. 
It then follows that \( \Pr( \Pd(B,E), \kbll, w) = .7 \), and \( \Pr(\Pg(A,B,7) \mid e, \kbll, w) = .25 \), where \( e = \Pt(A,B) \land \Piq(A,L) \land \Pd(A,E). \)

\end{example}

\section{Abstraction Framework} %
\label{sec:formulating_abstraction}

We assume that the abstraction framework is realized in terms of two types of representations: a \emph{high-level/abstract theory} that is mapped to a pre-existing  \emph{low-level/concrete theory}. Essentially,  the logical symbols (predicates and constants) may differ arbitrarily between the two theories.  In terms of notation, we use the subscript \( h \) to refer to components of the high-level theory, and \( l \) to refer to that of the low-level theory.

The first step is to formally establish the construct of a  \emph{refinement mapping} between the two theories: the mapping associates each high-level \emph{atom} to a low-level \emph{formula}, which may be arbitrarily complex. 

\begin{definition} Suppose \( \kbh \) and \( \kbl \) are two theories. We say \( m \) is a  \emph{refinement mapping} from \( \kbh \) to \( \kbl \) iff for all high-level atoms \( p \in \lang(\kbh), \) \( m(p) = \theta \sub {p} \) for some \( \theta\sub {p} \in \lang(\kbl) \).\footnote{When the high-level and low-level theories are defined over the same domain of discourse \( D \), \( m \) can have a compact specification of the form \( m(P(\vecx)) = \theta\sub P (\vecx) \), where \( P(\vecx) \) is a non-ground predicate, and \( \vecx \) are the only free variables in \( \theta\sub P \). So effectively the mapping works by  substitutions: for every  instance \( P(\vec c) \), we have \( m(P(\vec c)) = \theta\sub P(\vec c), \)  where \( \theta\sub P(\vec c) \) is obtained from \( \theta\sub P (\vecx) \) by substituting the free variables \( \vec x \) by \( \vec c. \)
	% Such a mapping scheme would closely tie in with symmetric probabilistic relational models \cite{symmetric-weighted-first-order-model}, where the weight function is defined in the manner that every instance of a predicate \( P(\vec x) \) obtains the same weight.
	} 
	
	The mapping \( m \) is assumed to extend to complex formulas  \( \phi  \in \lang(\kbh) \) inductively: for atoms \( \phi  = p \), \( m(\phi) \) is as  above; \( m(\neg \phi) = \neg m(\phi) \); \( m(\phi \land \psi) = m(\phi) \land m(\psi) \).

\end{definition}

It is worth noting that a mapping is deliberately asymmetrical in the sense that its range need not include all the atoms of the low-level theory. That is, there may be atoms \( q \in \lang(\kbl) \), and consequently, also constants and relations, that do not appear in \( m(p) \) for every \( p \in \lang(\kbh). \) After all, abstractions are about omitting irrelevant details.

In general, we will want to use these mappings to discuss model-theoretic properties of the two theories, so we introduce the notion of an isomorphism: 

\begin{definition} Given a refinement mapping \( m \) as above, we say that \( M\high \in \mods(\kbh) \) is \( m \)-isomorphic to \( M\low \in \mods(\kbl) \) iff for all  atoms \(  p \in \lang(\kbh), \) we have \( M\high \models p \) iff \( M\low \models m(p) \). We write this as \( M\high \iso M\low. \) 
	
\end{definition}

Thus, isomorphism provides a way to align the truth values between high-level atom and low-level formulas. In particular, because of how refinement mappings can be defined for complex formulas, we obtain the following property: 

\begin{theorem}\label{thm iso} Suppose \( \mh \iso \ml. \) Then for all \(  \phi \in \lang(\kbh), \) \( \mh \models \phi \) iff \( \ml\models m(\phi) \). 	
\end{theorem}

\begin{proof}
 We prove by induction on \( \phi. \) Base case immediate by definition. Negation: \( \mh \models \neg \phi \) iff \( \mh \not \models \phi \) iff (by hypothesis) \( \ml \not\models m(\phi) \) iff (by semantics) \( \ml \models \neg m(\phi) \) iff (by definition) \( \ml \models m(\neg \phi) \). Conjunction: \( \mh \models \phi \land \psi \) iff \( \mh \models \phi \) and \( \mh \models \psi \) iff (by hypothesis) \( \ml \models m(\phi) \) and \( \ml \models m(\psi) \) iff (by semantics) \( \ml \models m(\phi) \land m(\psi) \) iff (by definition) \( \ml \models m(\phi\land \psi). \)   
\end{proof}

\newcommand{\m}{m\sub {\U}}%

\begin{example} For the university PRM, we provide a mapping \( \m \) below. When free variables appear, we take it to mean that the mapping applies to all  substitutions. So, let $\m$ map \( \Pd(x,E) \), \( \Pt(x,y), \Piq(x,L)  \)   from \( \kbhh \) to the same atoms in \( \kbll, \) \( \m(\Pd(x,\Cn)) = \Pd(x,\Cm) \lor \Pd(x,\Ch)  \), \( \m(\Pg(x,y,\Cb)) = \Pg(x,y,5) \lor \Pg(x,y,6) \), \( \m(\Pg(x,y,\Co)) = \Pg(x,y,7) \lor \Pg(x,y,8)  \), and \( \m(\Pg(x,y,\Cg)) = \Pg(x,y,9) \lor \Pg(x,y,10) \). 

Suppose the domain includes a single student \( A, \) who takes course \( B \). Suppose \( M\sub h \) is a model of \( \kbhh \) where \( \{ \Piq(A,L), \Pt(A,B),  \) \( \Pd(B,\Ce), \Pg(A,B,\Co) \} \) holds. Now consider the model \( M\sub l \) of \( \kbll \) where \\ \( \{ \Piq(A,L), \Pt(A,B),  \Pd(B,\Ce), \Pg(A,B,7) \} \) holds. It is easy to verify that \( M\sub h \iso \ml \), because the main question is whether \( \ml \) satisfies \( \m(\Pg(A,B,\Co)) = \Pg(A,B,7) \lor \Pg(A,B,8) \), which it does.

\end{example}
In the following sections, we will discuss the properties of abstractions based on  mappings and isomorphisms.

\section{Unweighted Abstractions} %
\label{sec:unweighted_abstractions}

To obtain intuitions about the properties of abstract models from first principles, we will consider a fundamental type of abstraction: the absence of probabilities.\footnote{Thus, this section can  be seen to  establish a framework for abstraction in classical (unweighted) model counting.} In so much as probabilistic assertions quantify the likelihood of  worlds, omitting probabilities still informs us about the possible and the certain, thus allowing us to test whether  \( \kbh \) is consistent with \( \kbl. \)

\begin{definition} Given a weighted theory \( (\kb,w) \), the unweighted setting refers to the case when for all atoms \( p \in \lang(\kb) \), we have \( w(p) = w(\neg p) = 1. \) 
	
\end{definition}

Since probabilities do not occur in the setting, we can establish consistency by checking whether all conclusions by \( \kbh \) (that is, certain events) are also conclusions by \( \kbl \): in other words, are the conclusions \emph{sound}? We define: 

\begin{definition}\label{defn sound abs} The theory \( \kbh \) is a \emph{sound abstraction} of \( \kbl \) relative to refinement mapping \( m \) iff for all \(  \ml \in \mods(\kbl) \), there is a \(  \mh \in\mods(\kbh) \) such that \( \mh \iso \ml. \) 
	
\end{definition}

\begin{theorem}\label{thm sound abs} Suppose \( \kbh \) is a sound abstraction of \( \kbl \) relative to \( m. \) Then for all \(  \phi \in \lang(\kbh) \): \\
	(a) if \( \pr(m(\phi),\kbl,\wl) \gt 0 \) then \( \pr(\phi,\kbh,\wh) \gt 0 \); and (b) if \( \pr(\phi,\kbh,\wh) = 1\) then \( \pr(m(\phi),\kbl,\wl) = 1. \) 
	
\end{theorem}

\begin{proof}
 For (a), suppose the antecedent holds, which means there is a \( M\sub l \in \mods(\kbl) \) such that \( M\sub l \models m(\phi). \)
By assumption, there is a \( \mh\in \mods(\kbh) \) such that \( \mh \iso \ml, \) so \( \mh \models \phi, \) and \( \pr(\phi,\kbh,\wh) \neq 0. \)  (In the unweighted setting, the weight of \( \mh \) cannot be 0, since literals cannot get a 0 weight.) 

 For (b), suppose antecedent, but not consequent. That is only possible when there is a \( \ml \in \mods(\kbl) \) such that \( \ml\not\models m(\phi) \). But by assumption, there must be \( \mh \in\mods(\kbh) \) such that \( \mh \iso \ml. \) So, \( \mh \not\models \phi \) by Theorem \ref{thm iso}. Thus, \( \pr(\phi,\kbh,\wh) \neq 1 \). Contradiction. \end{proof}

\newcommand{\mm}{m\sub {\C}}%
\begin{example} It is easy to check that for the university PRM, \( \kbhh \) is a sound abstraction of \( \kbll \) wrt \( \m. \) 
	
	It is fairly straightforward to construct trivially unsound abstractions. 
To see a less obvious example, consider 
\( \kblll \) from before, and suppose it  also included: \( {\it CS}(x) \supset {\it Programming}(x) \) and \( {\it Physics}(x) \supset {\it Fieldwork}(x). \) And as discussed, let \( \kbhhh \) be a high-level theory consisting of the same sentences but with the predicate \emph{Science(x)} used everywhere instead of   \emph{CS(x)} and \emph{Physics(x)}. 

% Putting it together, then, \( \kblll \) would be the union of \( CS(x) \supset diff(x,H), Physics(x) \supset diff(x,E), AI(x) \supset CS(x), Astronomy(x) \supset Physics(x), {\it CS}(x) \supset {\it Programming}(x)   \) and \( {\it Physics}(x) \supset {\it Fieldwork}(x). \) In contrast, \( \kbhhh \) would be the union of \( CS(x) \supset diff(x,H), Physics(x) \supset diff(x,E), AI(x) \supset CS(x), Astronomy(x) \supset Physics(x), {\it CS}(x) \supset {\it Programming}(x)   \) and \( {\it Physics}(x) \supset {\it Fieldwork}(x). \)

Suppose \( B \) is a \emph{CS}-course.  Suppose \( \mm \) is a mapping that replaces  \emph{Science(x)} by \( {\it CS}(x) \lor {\it Physics}(x) \), but maps every other predicate to itself. Then, we have \( \pr(\phi \mid e,\kbhhh,\wh) = 1 \) for \( \phi =  {\it Programming}(B) \land {\it Fieldwork}(B) \) and \( e = {\it Science}(B) \), whereas, \( \pr(m(\phi) \mid m(e), \kblll, \wl) \neq 1 \), because 
there will be possible worlds where \( {\it CS}(B) \land \neg {\it Fieldwork}(B).  \)

\end{example}

Sound abstractions ascertain that conclusions by \( \kbh \) are consistent with \( \kbl. \) What about events considered \emph{possible} by \( \kbh \)? Because we are omitting information when constructing an abstract model, it may be that \( \kbh \) entertains an event as possible even though \( \kbl \) does not. 

\begin{definition}\label{defn complete abs} The theory \( \kbh \) is a \emph{complete abstraction} of \( \kbl \) relative to \( m \) iff for all \( \mh \in \mods(\kbh) \), there is a \( \ml \in \mods(\kbl) \) such that \( \mh \iso \ml. \) 
	
\end{definition}

\begin{theorem}\label{thm complete abs} Suppose \( \kbh \) is a complete abstraction of \( \kbl \) relative to \( m. \) Then for all \( \phi \in \lang(\kbh) \): (a) if \( \pr(\phi,\kbh,\wh) \gt 0 \) then \( \pr(m(\phi),\kbl,\wl) \gt 0 \); and (b) if \( \pr(m(\phi),\kbl,\wl) = 1 \) then \( \pr(\phi,\kbh,\wh) =1. \)
	
\end{theorem}

	\begin{proof}
	For (a), suppose antecedent. Then there is a \( \mh \in \mods(\kbh) \) such that \( \mh \models \phi \). By assumption, there is a \( \ml \in \mods(\kbl) \) such that \( \mh\iso \ml \) and so \( \ml \models m(\phi) \), and  \( \pr(m(\phi), \kbl, \wl) \neq 0. \) 
	
	For (b), suppose antecedent but not consequent. Then, there is a \( \mh \in \mods(\kbh) \) such that \( \mh \not\models \phi \). But by assumption, there is a \( \ml \in \mods(\kbl) \) such that \( \mh \iso \ml \), and so \( \ml \not\models m(\phi) \) by Theorem \ref{thm iso}. Thus, \( \pr(m(\phi),\kbl,\wl) \neq 1 \). Contradiction. \end{proof}

% Put differently, an exhaustive (but perhaps impractical) way to verify whether \( \kbh \) is a sound (or complete) abstraction is to verify that  the properties discussed in Theorem \ref{thm sound abs} (or \ref{thm complete abs} respectively) hold.

\begin{example}\label{ex:incomplete abstraction} The university PRM can be seen as a complete abstraction wrt \( \m \). 
	
To see a case where it is not complete, consider a variant high-level theory  \( \kbhh' \)  where we ignore  the difficulty of courses and  have only one rule: $\Piq(x,\Cl) \land \Pt(x,y) \supset \Pg(x,y,u)$ where \( u \in \{\Cb,\Co,\Cg\}. \) 
% \begin{itemize}
% 	\item[] $\Piq(x,\Cl) \land \Pt(x,y) \supset \Pg(x,y,u)$ where \( u \in \{\Cb,\Co,\Cg\}. \)
% \end{itemize}
Suppose the low-level theory is \( \kbll ' =  \Pd(B,\Ch) \land  \kbll, \) and \( A \) is a low-IQ student who takes \( B. \) It is easy to see that \( \Pr(\phi,\kbhh',\wh) \gt 0 \)  for \( \phi = \Piq(A,\Cl) \land \Pt(A,B) \land \Pg(A,B,\Cg) \), because \( \kbhh' \) says that any of the three grades levels are possible. But clearly, \( B \) being a hard course means that \( \Pd(B,\Ch) \land \m(\phi) \) cannot be satisfiable, and so it is a zero-probability event wrt \( \kbll'. \) 
	
\end{example}

\begin{definition} The theory \( \kbh \) is a \emph{sound and complete abstraction} of \( \kbl \) relative to \( m \) iff \( \kbh \) is both a sound and a complete abstraction of \( \kbl \) relative to \( m. \)
	
\end{definition}

\begin{theorem}\label{thm sound complete abs} Suppose \( \kbh \) is a sound and complete abstraction of \( \kbl \) relative to \( m. \) Then for every \( \phi\in\lang(\kbh) \), (a) \( \pr(\phi,\kbh,\wh) \gt 0 \) iff \( \pr(m(\phi),\kbl,\wl) \gt 0 \); and (b) \( \pr(\phi,\kbh,\wh) = 1 \) iff \( \pr(m(\phi),\kbl,\wl) = 1. \)
\end{theorem}

\begin{proof}
 Follows  from Theorems \ref{thm sound abs} and \ref{thm complete abs}.   	
\end{proof}

\section{Weighted Abstractions} %
\label{sec:weighted_abstractions}

Clearly the above theorems would not hold in general when considering non-trivial weights. It is easy to imagine a weight function that redistributes weights such that zero probability events in \( \kbl \) have high probabilities in \( \kbh \), and vice versa. So, outside the case of probabilities mapping exactly between \( \kbh \) and \( \kbl \) (discussed in the next section), we  need to understand how to abstract weighted theories. The previous section provided a recipe for abstractions, from which properties discussed in Theorems \ref{thm sound abs} and \ref{thm complete abs} followed. 
To a first approximation, then, we can motivate a definition for weighted abstractions by requiring that those properties hold categorically, in the form of  \emph{constraints}. But it turns out, we can do better. We can show that if the property about probable events hold as a constraint wrt a sound or complete abstraction, then the corresponding property about certain events follows as a consequence. (Recall that this duality is not about an event and its negation, which would follow from the axioms of probability, but about how the high-level and low-level theories align.)

To prepare for this approach, let us begin with a few properties that follow from the axioms of probability \cite{reasoning-about-knowledge-and-probability}, but are  established here using  WMC: 

\begin{theorem}\label{thm wmc properties} Suppose  \( (\kb,w) \) is a weighted theory. Then the following hold for all \( \phi, \psi \in \lang(\kb) \): \begin{enumerate}
	\item If \(\kb \models \phi  \) then \( \pr(\phi,\kb,w)  = 1 \).
	\item If \( \phi\land \kb \) is not satisfiable, then \( \pr(\phi,\kb,w) =0. \) 
	\item \( \pr(\neg \phi,\kb,w) = 1 - \pr(\phi,\kb,w) \). 
	\item \( \pr(\phi \lor \psi,\kb,w) = \pr(\phi,\kb,w) + \pr(\psi,\kb,w) - \pr(\phi\land\psi,\kb,w) \). 
	\item If \( \pr(\phi, \kb,w) = 0 \) then \( \pr(\phi \land \psi, \kb,w) = 0. \) 
	\item If \( \pr(\phi,\kb,w) \gt 0 \) then \( \pr(\phi\lor \psi, \kb,w) \gt 0. \)
	\item \( \pr(\phi, \kb, w) \geq \pr(\phi\land \psi, \kb, w) \). 
\end{enumerate}
	
\end{theorem}

\begin{proof} 
Proofs for (1) and (2) are already discussed in Theorem  \ref{prop:wmc entailment}. 
For (3), we use the fact that \( \mods(\kb) = \mods(\kb\land\phi) \cup \mods(\kb\land \neg\phi) \), and  \( | \mods(\kb) |  = | \mods(\kb\land\phi) | + | \mods(\kb\land \neg\phi)  | \). 
For (4), we use \( \mods((\phi \lor \psi)\land \kb) = \mods(\phi\land \kb) \cup \mods(\psi\land \kb)  \) but \( | \mods((\phi \lor \psi)\land \kb) | = | \mods(\phi\land \kb) | + | \mods(\psi\land \kb) | - | \mods(\psi\land\phi\land\kb) |. \) For (5), we see that \( \kb \land \phi \) has no model (or only  zero weight models), and so that clearly also holds for \( \kb \land \phi\land\psi. \) For (6), the models for \( \kb \land \phi \) yield a non-zero probability, and these are clearly included in the models for \( \kb \land (\phi \lor \psi). \) For (7), the models of \( \phi\land \psi \) must be a subset (not necessarily proper) of the models of \( \phi. \) \end{proof}

\begin{definition}\label{defn weighted sound abs} The theory \( (\kbh,\wh) \) is a \emph{weighted sound abstraction} of \( (\kbl,\wl) \) relative to refinement mapping \( m \) iff \( \kbh \) is a sound abstraction of \( \kbl \) relative to \( m \), and for all \( d \in \lits(\kbh) \), if \( \pr(m(d), \kbl,\wl) \gt 0 \) then \( \pr(d,\kbh,\wh) \gt 0. \)
\end{definition}

We will now show that this stipulation at the level of literals  immediately implies  the validity of the constraint for all formulas: 

\begin{theorem}\label{thm weighted sound abs extension for formulas} Suppose \( (\kbh,\wh) \) is a weighted sound abstraction of \( (\kbl,\wl) \) relative to \( m. \) Then for all \( \phi \in \lang(\kbh) \), if \( \pr(m(\phi), \kbl,\wl) \gt 0 \) then \( \pr(\phi,\kbh,\wh) \gt 0. \)
	
\end{theorem}

\begin{proof} 
By induction on \( \phi. \) The case of atoms and negations is immediate by definition. So we only need an argument for disjunctions. Suppose \( \pr(m(\phi\lor\psi), \kbl,w) \gt 0 \), that is, by definition, \( \pr(m(\phi) \lor m(\psi), \kbl,w) \gt 0 \). By Theorem \ref{thm wmc properties} (4), \( \pr(m(\phi), \kbl,w) + \pr(m(\psi), \kbl,w) - \pr(m(\phi) \land m(\psi), \kbl, w) \gt 0. \) This is of the form \( x + y - z \gt 0, \) where \( x, y, z \geq 0 \) since these are probabilities. We have 3 cases. 
	
Case $x=0$: We note that \( z=0 \) too, by Theorem \ref{thm wmc properties} (5). 
So \( y \gt 0. \) By hypothesis, \( \pr(\psi,\kbh,\wh) \gt 0 \), and therefore \( \pr(\phi\lor\psi,\kbh,\wh) \gt 0 \) by Theorem \ref{thm wmc properties} (6). 

Case \( y= 0\): Symmetric to \( x=0. \) 

Case \( x\neq 0 \) and \( y \neq 0 \): By  hypothesis, \( \pr(\phi, \kbh,\wh) \gt 0 \) and \( \pr(\psi,\kbh,\wh) \gt 0 \). Even if \( \pr(\phi\land\psi,\kbh, \wh) \gt 0 \), by Theorem \ref{thm wmc properties} (7), it must be that it must be smaller or equal to the other probabilities. (That is, if \( a, b,c \gt 0 \), \( c \leq a \) and \( c\leq b \), then \( a+b-c \gt 0. \)) So, \( \pr(\phi \lor \psi, \kbh,\wh) \gt 0. \)   
\end{proof}

The key result of this definition is that the property on certain events, seen in Theorem \ref{thm sound abs} follows as a consequence: 

\begin{theorem}\label{thm weighted sound abs} Suppose \( (\kbh,\wh) \) is a weighted sound abstraction of \( (\kbl,\wl) \) relative to \( m. \) Then for all \( \phi\in \lang(\kbh) \), if \( \pr(\phi,\kbh,\wh) = 1 \) then \( \pr(m(\phi), \kbl,\wl) = 1. \) 
\end{theorem}

\begin{proof} 
Suppose antecedent but not consequent. Then there is some \( \ml \in\mods(\kbl) \) such that  \( \ml \not\models m(\phi) \) and it has non-zero weight. (If all such \( \ml \) have zero weight, then the consequent cannot be falsified because these models do not influence the probability.) By assumption, there is a \( \mh \in \mods(\kbh) \) such that \( \mh \iso \ml, \) and so \( \mh \not\models \phi \). 

There are now two cases, depending on the weight of the model \( \mh \). (And so the proof deviates from that for Theorem \ref{thm sound abs}.)	

Case \( \wh(\mh) \neq 0 \): The proof follows as in Theorem \ref{thm sound abs}, yielding a contradiction.

Case \( \wh(\mh) = 0 \): Let \( \mh\su\downarrow \) be a formula denoting the conjunction of the literals true at \( \mh. \) (Since there are finitely many atoms, such a formula can be obtained.) Because \( \mh \iso \ml, \) \( \ml \models m(\mh \su \downarrow). \) Overloading the notation \( M\su\downarrow \) to mean  conjunction and set of literals true at \( M \), \( m(\mh\su\downarrow ) \subseteq \ml\su\downarrow, \) the latter being the set of literals true at \( \ml. \) But by assumption \( \ml \) has non-zero weight, which means \( \pr(\ml\su\downarrow, \kbl,\wl) \gt 0 \). It follows that \( \pr(m(\mh\su\downarrow ), \kbl,\wl) \gt 0 \), because otherwise Theorem \ref{thm wmc properties} (5) would be contradicted. By Theorem \ref{thm weighted sound abs extension for formulas}, \( \pr(\mh\su\downarrow, \kbh,\wh) \gt 0 \), and so \( \wh(\mh) \neq 0 \). Contradiction. \end{proof}

\begin{example} The university PRM can be seen to be a weighted sound abstraction wrt \( \m \). 
	
	Consider the university PRM  with a variant  high-level theory \( \kbhh''  \), where the third constraint is the following instead: \begin{itemize}
	\item[{\bf 1}] \( \Piq(x,L) \land \Pd(y,E) \land \Pt(x,y) \supset \Pg(x,y,G)  \)
\end{itemize}
Consider the query \( \phi = \Piq(A,L) \land \Pd(B,E) \land \Pt(A,B) \supset \Pg(A,B,O) \). Clearly, the low-level theory accords a non-zero probability to \( \m(\phi) \), but because of the third constraint, \( \kbhh'' \) accords a zero probability to \( \phi. \) Thus, this is not a sound weighted abstraction. 
	
\end{example}

Following these results, extending complete abstractions as well as sound and complete abstractions is analogous, which we state here for the sake of completeness. (The proofs are also analogous and hence omitted.)

\begin{definition}\label{defn weighted complete abs} The theory \( (\kbh,\wh) \) is a \emph{weighted complete abstraction} of \( (\kbl,\wl) \) relative to refinement mapping \( m \) iff \( \kbh \) is a complete abstraction of \( \kbl \) relative to \( m \), and for all \( d \in \lits(\kbh) \), if \( \pr(d, \kbh,\wh) \gt 0 \) then \( \pr(m(d),\kbl,\wl) \gt 0. \) 
\end{definition}

\begin{theorem}\label{thm weighted complete abs} Suppose \( (\kbh,\wh) \) is a weighted complete abstraction of \( (\kbl,\wl) \) relative to \( m. \) Then for all \( \phi\in \lang(\kbh) \), (a)
	if \( \pr(m(\phi),\kbl,\wl) = 1 \) then \( \pr(\phi, \kbh,\wh) = 1 \); and (b) if \( \pr(\phi,\kbh,\wh) \gt 0 \) then \( \pr(m(\phi),\kbh,\wh) \gt 0. \)
\end{theorem}

\begin{example}\label{ex:weighted incomplete} The university PRM can be seen to be a weighted complete abstraction wrt \( \m \).
	
	Example \ref{ex:incomplete abstraction} also applies as an instance of an abstraction that is not weighted complete via:  \begin{itemize}
	\item[{\bf .33}] $\Piq(x,\Cl) \land \Pt(x,y) \supset \Pg(x,y,u)$ where \( u \in \{\Cb,\Co,\Cg\}. \) 
\end{itemize} 
Mainly because the difficulty of courses is ignored, an event is considered probable by the high-level theory but not by  the low-level one. 
	
\end{example}

\begin{definition}\label{defn weighted sound complete abs} The theory \( (\kbh,\wh) \) is a \emph{weighted sound and complete abstraction} of \( (\kbl,\wl) \) relative to refinement mapping \( m \) iff it is both a weighted sound and a weighted complete abstraction. 
\end{definition}

\begin{theorem}\label{thm weighted sound complete abs} Suppose \( (\kbh,\wh) \) is a weighted sound and complete abstraction of \( (\kbl,\wl) \) relative to \( m. \) Then for all \( \phi\in \lang(\kbh) \), (a) \( \pr(m(\phi),\kbl,\wl) = 1 \) iff \( \pr(\phi, \kbh,\wh) = 1 \); and (b) \( \pr(m(\phi),\kbl,\wl) \gt 0 \) iff \( \pr(\phi,\kbh, \wh) \gt 0. \)
\end{theorem}

\begin{proof}
Follows as a corollary from the results on weighted sound, and weighted complete abstractions. 
\end{proof}

\section{Exact Abstractions} %
\label{sec:full_alignment}

The most faithful  case of aligning the high-level and low-level theories is when the probabilities coincide for all high-level queries.\footnote{The distribution on the high-level theory is essentially a ``push-forward"  measure  \cite{introduction-to-real-analysis}.}

\begin{definition}\label{defn weighted exact} The theory \( (\kbh,\wh) \) is a \emph{weighted exact abstraction} of \( (\kbl,\wl) \) relative to refinement mapping \( m \) iff \( \kbh \) is a sound and complete abstraction of \( \kbl \) relative to \( m \), and 
		for all \( \phi \in \lang(\kbh) \), \( \Pr(\phi, \kbh, \wh) = \Pr(m(\phi), \kbl, \wl) \).
\end{definition}

\begin{example} The university PRM can be seen to be an instance of a weighted exact abstraction wrt \( \m \). 
	
	In contrast, the variant in Example \ref{ex:incomplete abstraction}/\ref{ex:weighted incomplete} does not belong to this type because the high-level theory accords a  probability of 1/3 to a low-IQ student taking a difficult course and still getting a good grade, whereas the low-level theory considers that improbable.

\end{example}

\section{Abstracting Evidence} %
\label{sec:explanations}

Recall that we can query \( \phi \) wrt evidence \( e \) for theory \( (\kb,w) \) using : \[
	\pr(\phi\mid e, \kb, w) =  \frac{\wmc(\phi\land e\land \kb, w)}{\wmc(e\land \kb, w)} = \frac{\pr(\phi \land e,\kb,w)}{\pr( e,\kb,w)} 
\]
We assumed so far that \( \phi, e \in \lang(\kb) \). 
However, in many applications needing  abstraction, it is often the case that  observations are low-level (e.g., readings on sensor), whereas the query is at the high-level (e.g., interactions with user). In this section, we discuss some ways  to reconcile this issue.\footnote{It is conceivable that there may be other approaches for this reconciliation, and in our inquiry as well, it will become clear  that  a number of variants present themselves.  We also limit the discussion to  exact abstractions for simplicity. 
} 

\newcommand{\conc}{m\su {-1}}

Consider  low-level evidence \( e \in \lits(\kbl) \). For simplicity, let \( e \) be a literal. Without loss of generality, let mappings be in conjunctive normal form (CNF). 
We say a literal is \emph{pure} in a CNF \( \theta \) if its complement  does not appear in \( \theta. \) (E.g., \( p \) is pure in \( p\lor q \) but not in \( \neg p \lor q \); in contrast, \( \neg p \) is pure in  \( \neg p \lor q \) but not in  \( p\lor q \).) We observe that, by construction, there may be many high-level atoms that map to formulas involving \( e. \)
So, given a mapping \( m, \) let us retrieve these by \emph{concretization}: \[\begin{array}{l}
	\conc (e) = \{ \textrm{atom } p \in \lang(\kbh) \mid   e \textrm{ is mentioned \& pure in   }m(p) \}.
\end{array}
	\]
(That is, $m(p)$ is a CNF formula.)
Here, \( \conc(e) \) is equivalently expressed as a formula: \( \bigvee p\sub i  \).  The idea is that by looking at high-level atoms where \( e \) is pure under the mapping, we are essentially finding atoms that agree with the evidence (and not its negation).

We can now retrieve all low-level  sentences these map to by re-applying \( m \) as follows:  \( m(\conc (e)) = \bigvee m(p\sub i ). \) (It is easy to see that \( e \) will remain pure in \( m(\conc(e)) \).)

An immediate case, then, of  conditioning   being  straightforward is when \( e = m(\conc(e)) \): 

\begin{theorem}\label{thm exact evidence} Suppose \( (\kbh,\wh) \) is a weighted exact abstraction of \( (\kbl,\wl) \) relative to \( m. \) Suppose \( e \in \lits(\kbl) \) and \( e = m(\conc(e)). \) Then for any \( \phi \in \lang(\kbh), \) \( \pr(\phi \mid \conc(e), \kbh, \wh) = \pr(m(\phi) \mid e, \kbl, \wl).  \)
	
\end{theorem}

\begin{proof} 
By assumption, the probability of \( \phi \land \conc(e) \) wrt \( \kbh \) must be the same as that of \( m(\phi) \land m(\conc(e))  \) at the low-level.   
\end{proof}

A simple example is the case of \( \Pd(x,E) \) in the university PRM, as it was mapped to the same atom at both levels. 

But beyond this simple case, it is not always possible to reason about low-level events in an exact manner at the high-level. Indeed, as mentioned before, omitting details is the very goal of abstraction. For example, in the university PRM, given any course \( B, \) \( \pr(\Pd(B,M), \kbll, \wl) = .1,  \) but clearly there is no way to syntactically arrange \( \set{\Pd(B,E), \Pd(B,N)} \) in \( \kbhh \) to obtain that number. Of course, it would not be hard to show  a more involved property, such as \( \pr(\Pd(B,N), \kbhh, \wh) \geq \pr(\Pd(B,M), \kbll, \wl) \).  

\newcommand{\concc}{\m\su {-1}}

Rather than treating such  properties, we will consider the case where probabilities can correspond exactly. Then, one way to incorporate low-level evidence is to \emph{weaken} it, in the sense that  conditioning wrt the low-level theory would suffer from a loss in detail, which is precisely the problem faced by the high-level theory. We may think of using \( m(\conc(e)) \), for example. 
However, that is not sufficient for conditioning to be \emph{correct}, because \( m(\conc(e)) \) can say \emph{more} and \emph{less} than \( e. \) For example, in the university PRM, suppose we have evidence \( e = \Pd(B,M) \) for \( \kbll \). So \( \m\su {-1}(e) = \Pd(B,N) \), and  \( \m(\concc(e)) = \Pd(B,M) \lor \Pd(B,H) \), which is saying less than \( e. \) This is reasonable. But suppose for the sake of the argument, \( \m(\concc(e)) = (\Pd(B,M) \lor \Pd(B,H)) \land \Pd(C,H) \). (This is somewhat artificial but well-defined.)
The problem  is that \( e \) does not \emph{imply} anything about the difficulty of course \( C. \) Thus, if we use   \( \m(\concc(e))\) as evidence, we will be falsely assuming   facts that were not observed.

To get around this, we stipulate this implication formally:

\begin{definition} Given evidence \( e \) and mapping \( m \), we define the \( m \)-weakening of \( e \) as \( m(\conc(e)) \). It is definable iff  \( e \models m(\conc(e)). \) 
	
\end{definition}

The most obvious (and reasonable) case where definability follows is when \( m(\conc(e)) \) is a clause, that is, a disjunction of literals. Because \( e \) is pure in \( m(\conc (e)) \), it immediately follows that \( e\models m(\conc (e)) \). (E.g., \( p \) is pure in \( p\lor q \), and of course \( p\models p\lor q. \))

\begin{theorem} Suppose \( (\kbh,\wh) \) is a weighted exact abstraction of \( (\kbl,\wl) \) relative to \( m. \) Suppose \( e\in \lits(\kbl) \) and its \( m \)-weakening is definable. Then, \( \pr(\phi \mid \conc(e), \kbh, \wh) = \pr(m(\phi) \mid m(\conc(e)), \kbl,\wl ) \). 
	
\end{theorem}

\begin{proof} 
Proof analogous to Theorem \ref{thm exact evidence}.   
\end{proof}

\begin{example} For the university PRM and \( e = \Pd(B,M) \), its  \( \m \)-weakening   is \( \Pd(B,M) \lor \Pd(B,H) \). And indeed, \( e \models \m(\concc(e)).\) For the query \( \phi = \Piq(A,L) \land \Pt(A,B) \land \Pg(A,B,O) \), its probability given \( \concc(e) = \Pd(B,N) \) at the high-level coincides with the probability of \( \Piq(A,L) \land \Pt(A,B) \land (\Pg(A,B,7) \lor \Pg(A,B,8)) \) given \( \m(\concc(e)) \) at the low-level. 
	
\end{example}

\section{Weak Exact Abstractions} % (fold)
\label{sec:weak_abstractions}

Definition \ref{defn weighted exact} naturally generalizes previous definitions on weighted and unweighted abstractions in terms of first stipulating that the logical representations align, and then insisting that probabilities match. One might wonder, of course, why stipulate the former, and not simply rely on the latter, which would make the treatment much simpler. In particular, we put forward a definition of \emph{weak exact abstractions}: 

\begin{definition}\label{defn weak exact} The theory \( (
	\kbh,\wh) \) is a \emph{weak exact abstraction}  of \( (\kbl,\wl) \) relative to refinement mapping \( m \)  iff for all \( \phi\in\lang(\kbh) \), \( \pr(\phi,\kbh,\wh) = \pr(m(\phi),\kbl,\wl) \). 
	
\end{definition}
% (Of course, since sound and complete abstractions correspond closely to the alignment of probable and improbable events, formulating a notion of ``non-exact"  weak abstractions reduces to what was already discussed to Section \ref{sec:weighted_abstractions}.)
% (Analogous definitions can be put forward for non-exact abstractions too.) 
It is immediate to see that since the constraint for weak exact abstraction is embedded in the notion of a weighted exact abstraction, if a high-level theory is a weighted exact abstraction then it also a weak exact abstraction. 

\begin{proposition} Suppose \( (\kbh,\wh)\) and \( (\kbl,\wl ) \) are theories and \( m \) is a refinement mapping. Suppose \( (
	\kbh,\wh) \) is a weighted exact abstraction  of \( (\kbl,\wl) \). Then \( (
	\kbh,\wh) \) is a weak exact abstraction  of \( (\kbl,\wl) \).
	
\end{proposition}

\begin{proof} By assumption, (a) \( 
\kbh \) is a sound and complete abstraction of \( \kbl \) relative to \( m, \) and (b) for all \( \phi\in\lang(\kbh) \), \( \pr(\phi,\kbh,\wh) = \pr(m(\phi),\kbl,\wl) \). Because of (b), the claim is immediate. \end{proof}

Needless to say, weak exact abstractions do not imply weighted exact abstractions, as there is no requirement that isomorphisms hold between the models of \( \kbh \) and \( \kbl. \) Formally: 

\begin{theorem} Suppose \( (\kbh,\wh)\) and \( (\kbl,\wl ) \) are theories and \( m \) is a refinement mapping. Suppose \( (
	\kbh,\wh) \) is a weak exact abstraction  of \( (\kbl,\wl) \) relative to \( m \). Then it does not follow that \( (
	\kbh,\wh) \) is a weighted exact abstraction  of \( (\kbl,\wl) \).
	
\end{theorem}

To prove this claim, it suffices to provide an example that is a weak exact abstraction but not a weighted exact abstraction. So any unsound and/or incomplete abstraction where the  probabilities of  high-level atoms are made to match the low-level mappings would do. 
Here is a particularly extreme case  involving an inconsistency: 

\begin{example} Let \( \kbl \) be any theory and let \( \kbh = \kbl \land (p \lor q) \), where \( \set{p,q} \) are fresh propositions not present in \( \kbl. \) Let \( \wl \) be any weight function for \( \kbl \) which maps atoms \( r \) in \( \kbl \) to a number in \( [0,1] \), with the understanding that \( \wl(\neg r) = 1 - \wl(r). \)
	Let $s$ be any atom in $\kbl$, and so $s$ is an atom in $\kbh$ by construction. 
Let \( m \) be a mapping that maps every atom in \( \kbl \) to itself, and \( m(p) = s \lor \neg s \) and \( m(q) = s\land \neg s \). Furthermore, for every atom \( r \in \kbh - \set{p,q}, \) let \( \wh(r) = \wl(r) \),  \( \wh(p) = 1 \) (so, \( \wh(\neg p) \) = 0) and \( \wh(q) = 0 \) (so, \( \wh(\neg q) = 1 \)). That is, \( \pr(p,\kbh,\wh) = 1 = \pr(m(p),\kbl,\wl), \pr(q,\kbh,\wh) = 0 = \pr(m(q),\kbl,\wl) \), and 
	the probabilities of all others atoms in \( \kbl \) are the same in \( \kbh. \) It is thus clear that \( (
	\kbh,\wh) \) is a weak exact abstraction  of \( (\kbl,\wl) \). In particular, for any \( \mh \in \mods(\kbh) \) such that \( \mh \models q \), we have \( \wh(\mh) = 0 \)   and by extension, letting \( \mh\su \downarrow \) be the formula denoting the conjunction of literals that are true at \( \mh, \) \( \pr(\mh\su\downarrow,\kbh,\wh) = 0  \). It follows also that  \(  \pr(m(\mh\su\downarrow),\kbl,\wl) = 0\). Analogously, for all models \( \mh \in \mods(\kbh) \) such that \( \mh \not\models q, \) by construction, we  have \(  \pr(\mh\su\downarrow,\kbh,\wh) = \pr(m(\mh\su\downarrow),\kbl,\wl)  \). 
	
However, by construction, there is a \( \mh \in \mods(\kbh) \) (for example,  one where \( \mh\models q \)) such that there is no \( \ml \in \mods(\kbl) \) where \( \ml \models m(q) \). So \( \kbh \) is not a sound and complete abstraction of \( \kbl \) relative to \( m. \) 

\end{example}

Thus, weighted exact abstractions is a stronger requirement than weak exact abstractions. Clearly, weak  exact abstractions would be much more  attractive as they involve fewer checks than weighted exact abstractions, the former only requiring probabilistic alignment whilst the latter also insisting on logical  alignment. 
So what is to be gained by the stronger requirement? 
The answer may depend on the application context. The stronger requirement   guarantees downward compatibility with results like Theorems \ref{thm sound complete abs} and \ref{thm weighted sound complete abs}, and so we can be assured about the correctness of the abstraction at the qualitative level. For example, when there is only partial knowledge about probabilities, one may express this knowledge in the form of constraints (e.g., the probability of event \( \alpha \) is \( \geq 0.4 \)), as in \cite{reasoning-about-uncertainty,reasoning-about-knowledge-and-probability,a-first-order-logic-of-probability-and-only}. In this case, establishing logical alignment may be worthwhile in the first instance, either until that partial knowledge is resolved, or  in addition to abstracting such constraints. Analogously, probabilities in most real-world applications are typically learnt from data in a parameter estimation step \cite{probabilistic-graphical-models:-principles}. In this case, either  because the data is not complete or because observations are obtained at run-time in an online setting, the  posteriors might change and so constructing weak exact abstractions 
may not be worthwhile. Indeed, if we expect the parameters of the low-level theory to change, we could consider unweighted abstractions so as to  construct a high-level theory that is not sensitive to weights. We would then  match probabilities for a particular training epoch with the added assurance  that at least the syntactic form of the high-level theory does not change from epoch to epoch.

\section{Deriving Abstractions} % (fold)
\label{sec:deriving_abstractions}

The main thrust of this paper is on the semantical properties of abstractions, formulated under the assumption that we are given the high-level and low-level theory and the appropriate refinement mapping. Based on these properties, we will now motivate a few directions for deriving abstractions automatically. 
These directions are to be seen as  schemas that appeal to exhaustive search, and so are not necessarily efficient. In particular, they identify general properties that hold, based on which special tractable cases, or variations, may be considered.

% These approaches mostly appeal to exhaustive search and so are not necessarily efficient; in other words, this exercise is best seen as an algorithmic schema, based on which special tractable cases may be considered.
 
% Before turning to this schema, however, let us consider simple yet useful cases where abstract atoms can replace complex formulas.

\subsection{Formula substitutions} % (fold)
\label{sub:syntactic_substitutions}

% We will first attempt to understand whether  syntactic substitutions can yield abstractions.
% 
Rather than deploying a general search procedure (as motivated in the following section),  
% we will consider appealing to syntactic substitutions for yield abstractions.
% %
it is arguably easier to consider syntactic substitutions where possible.
%
% A particularly attractive case of deriving abstractions is when  high-level atoms  can simply replace complex low-level formulas.
A simple yet useful case where correctness is not compromised is when complex formulas in the low-level theory appear in the same way everywhere, and so can be abstracted as a high-level atom. For example, suppose \( \kbl \) mentions the atoms \( \set{p\sub 1, \ldots, p\sub k, q, \ldots, r} \). For simplicity, suppose \( \kbl \) is in conjunctive normal form, that is, \( \kbl =  \phi \sub 1 \land \ldots \land \phi \sub n, \) where \( \phi\sub i \) are clauses, 
and let \( \lambda \) be a clause only mentioning \( \set{q,\ldots,r} \). Suppose for every \( i, \) either \( \phi\sub i  \) does not mention \( \set{q,\ldots,r} \) or \( \phi \sub i = \lambda \lor \psi\sub i, \) where \( \psi\sub i \) is a clause only mentioning \( \set{p\sub 1, \ldots, p\sub k}. \) In English: the symbols \( \set{q,\ldots,r} \) do not appear in \( \phi\sub i \) except as the clause \( \lambda. \) We can construct a high-level theory that replaces \( \lambda \) with a new atom \( t. \) So let \( \kbh \) be exactly like \( \kbl \) except that every instance of \( \lambda \) is replaced by \( t \). Clearly the refinement mapping \( m \) maps \( t \) to \( \lambda \) and all other atoms to themselves. 
 Further, let \( \pr(p\sub i, \kbh,\wh) = \pr(p\sub i, \kbl,\wl) \) and \( \pr(t,\kbh,\wh) = \pr(\lambda, \kbl,\wl) \). It is now not hard to see that \( (\kbh,\wh) \) is not only a weak exact abstraction of \( (\kbl,\wl) \) relative to \( m, \) but a weighted exact abstraction too.

\begin{proposition}\label{prop:formula subs} Suppose \( \kbh,\wh, \kbl,\wl,m \) are as above, and \( \lambda \) is satisfiable. Then \( (\kbh,\wh) \) is a  weighted exact abstraction of \( (\kbl,\wl) \) relative to \( m. \)
	
\end{proposition}

\begin{proof} We will first prove that for every \( \mh \in \mods(\kbh) \) there is a \( \ml \in \mods(\kbl) \) such that \( \mh \iso \ml \), and vice versa. So let \( \mh \) be any model of \( \kbh. \) 
By assumption,  \( \lambda \) is satisfiable, and suppose \( M' \) is one such satisfying assignment (note: \( M' \) is essentially a partial model  for the atoms in \( \lang(\kbl) \)). Let \( \ml \) be exactly like \( \mh \) in interpreting \( \set{p\sub 1, \ldots, p\sub k} \), and for the atoms in \( \set{q,\ldots,r} \), let \( \ml \) assign exactly as \( M' \) would. It now follows that for every atom \( u \in  \lang(\kbh) \), that is, \( u \in  \set{p\sub 1, \ldots, p\sub k, t} \), we have that \( \mh \models u \) iff \( \ml \models m(u) \). The case of finding a high-level model for every low-level model is analogous. Thus, \( \kbh \) is a sound and complete abstraction of \( \kbl \) relative to \( m. \) 

In terms of aligning probabilities, the case is immediate by construction. \end{proof}

A slight variant of this idea  could be 
used for abstracting multi-valued (or even continuous) random variables,  provided the  queries can be reasoned  as  Boolean variables (as seen  in   \cite{michels2016approximate,holtzen2017probabilistic,probabilistic-inference-in-hybrid-domains}). We demonstrate using an example below. 

\begin{example} Let us consider a random variable \( X \) that is uniformly drawn from \( \set{0,1,\ldots,9} \) and suppose we are only interested in queries about  \( X\geq 8 \) or its negation. In our terms, we could imagine \( X \in \set{0,1,\ldots,9} \)  to be represented using the formula \( p\sub 0 \lor \ldots \lor p\sub 9, \) and \( X\geq 8 \) to mean \( p\sub 8 \lor p\sub 9 \). More precisely,  suppose the low-level theory \( (\kbl,\wl) \) is one where \( \kbl \) contains the following formulas: \begin{itemize}
	\item \( p\sub 0 \lor \ldots \lor p\sub 9 \)
	\item \( p\sub 0 \equiv \neg (p\sub 1 \lor \ldots \lor p\sub 9), \ldots,  p\sub 9 \equiv \neg (p\sub 0\lor \ldots \lor p\sub 8) \)
\end{itemize}
and \( \wl \) is such that \( \Pr(p\sub 0, \kbl,\wl) = 0.1, \ldots, \pr(p\sub 9, \kbl,\wl) = 0.1. \) Thus \( \Pr(p\sub 8 \lor p\sub 9, \kbl,\wl) = 0.2 \).  
A high-level theory \( \kbh \) could be constructed as an abstraction containing the single atom \( q \) provided:\begin{itemize}
\item \( m \) is a mapping \( m(q) = p\sub 0 \lor \ldots \lor p\sub 7 \); 
	\item \( \wh \) is defined so that \( \pr(q,\kbh,\wh) = 0.8 \) as a result of  which \( \pr(\neg q, \kbh,\wh) = 0.2 \).
\end{itemize}
 % now be constructed that lumps the disjunction over \( \set{p\sub 0, \ldots, p\sub 7} \) as a single atom \( q \) with the understanding that \( \neg q \) would denote \( p\sub 8 \lor p\sub 9. \) Then, let \( \pr(q,\kbh,\wh) = 0.8 \) and \( \pr(\neg q, \kbh,\wh) = 0.2 \).
 It is now easy to see that \( (\kbh,\wh) \) is a weighted exact abstraction of \( (\kbl,\wl) \) relative to \( m \).

\end{example}

What is interesting about this example in relation to Proposition \ref{prop:formula subs} is that the \( \lambda \) \( =  p\sub 0 \lor \ldots \lor p\sub 7 \) in question is abstracted  as an atom \( q \)  as usual, but we are using \( \neg q \) to mean \( p\sub 8 \lor p\sub 9 \), which is not purely syntactical substitution. We are appealing to the  property that \( \kbl \models (p\sub 0 \lor \ldots \lor p\sub 7) \equiv \neg (p\sub 8 \lor p\sub 9) \) obtained by logical equivalence relative to \( \kbl. \) Thus, like in Proposition \ref{prop:formula subs}, we are replacing a low-level formula \( \lambda \) with a high-level atom \( t \) but unlike that setting, we are replacing a low-level formula \( \gamma \) with \( \neg t \) provided \( \kbl \models \gamma \equiv m(\neg t). \) It is not hard to show that a correctness result also holds in this extended setting.

  % that maps all atoms other than \( \set{p\sub 0,\ldots, p\sub 9} \) to themselves, and \( q \) to \( p\sub 0 \lor  \ldots\lor p\sub 7 \).

% subsection syntactic_substitutions (end) 

% and so this exercise is best seen as as a
% ; in other words, this exercise is best seen as an  algorithmic schema. Like in the previous section, we will restrict our attention to exact abstractions for simplicity.

\subsection{A Generic Search Algorithm} % (fold)
\label{sub:a_generic_algorithm}

% subsection a_generic_algorithm (end)

Let us now consider a generic search algorithm for deriving abstractions. We will begin with weak exact abstractions, as they are clearly simpler to treat than weighted exact abstractions. We return to the latter in a subsequent section. 
%
% so we focus on that case for simplicity. The constraints for the more robust weighted exact abstraction is more involved, including tests for properties asserted in, for example, Definition \ref{defn sound abs}. (In that regard, model enumeration strategies afforded by knowledge compilation \cite{a-knowledge-compilation-map} may prove particularly useful. Also see Section \ref{sec:beyond_weak_abstractions}.)
% The algorithm schema below can be easily adapted for this stronger requirement if need be, by swapping the correctness check with the appropriate  one.

The starting point here is that as input we are given the low-level theory \( (\kbl,\wl) \). We are then interested in constructing a high-level theory and  we assume to be also given the set of high-level predicates \[ P\sub h \su 1 (x, \ldots, y), \ldots, P\sub h \su k (x, \ldots, z) \] from which \( \kbh \) is to be constructed. 
If we now guess a \( (\kbh, \wh) \) and a mapping \( m \), we know from Definition \ref{defn weak exact} that testing for the following property would  ascertain that the current guess is a  weak exact abstraction: \begin{equation}\label{eq cond weighted exact}
	\text{for all \( \phi \in \lang(\kbh), \pr(\phi,\kbh,\wh) =\pr(m(\phi),\kbl,\wl) \) } \tag{\( \star \)}
\end{equation}
Thus, a general schema might then look as follows: 

\vspace{.4cm}

\begin{algorithm}[H]
\SetAlgoLined
\KwData{Low-level theory \( (\kbl,\wl), \{P\sub h \su 1 (x, \ldots, y), \ldots, P\sub h \su k (x, \ldots, z)\} \)}
\KwResult{\bf{success} / \bf{failure}}
 \( (\kbh,\wh,m) = \set{} \) \\ 
 \While{true}{
  Guess  \( (\kbh,\wh,m) \) that is different from previous guesses, where 
  \( \kbh \) 
  only uses the mentioned predicates \\ 
 \textbf{if}  \eqref{eq cond weighted exact} is true \textbf{then} \\
 \quad return {\bf success} \\ 
 \textbf{if} no more unique guesses \textbf{then} \\
 \quad return {\bf failure}
 }
 % \eIf{\eqref{eq cond weighted exact} is true}{return \bf{success}}\;
 % \eIf{no more unique guesses}{break}
 %  }
 %  return \bf{failure}
  % \eIf{\eqref{eq cond weighted exact} is true}{return \bf{success} }{return \bf{failure}}
 \caption{Guessing \( (\kbh,\wh) \) and refinement mapping \( m \)} 
 \label{alg:find one}
\end{algorithm}

% \begin{algorithm}[H]
% \SetAlgoLined
% \KwData{Low-level theory \( (\kbl,\wl), \{P\sub h \su 1 (x, \ldots, y), \ldots, P\sub h \su k (x, \ldots, z)\} \)}
% \KwResult{\bf{success} / \bf{failure}}
%  \( (\kbh,\wh) = \set{}, m = \set{} \)\;
%  \While{\eqref{eq cond weighted exact} is not true}{
%   Guess \( \kbh \) only using $\{P\sub h \su 1 (x, \ldots, y), \ldots, P\sub h \su k (x, \ldots, z)\}$\;
%   Guess \( \wh \) and \( m \)\; }
%   \eIf{\eqref{eq cond weighted exact} is true}{return \bf{success} }{return \bf{failure}}
%  \caption{Guessing \( (\kbh,\wh) \) and refinement mapping \( m \)}
%  \label{alg:find one}
% \end{algorithm}

\vspace{.4cm}

It is not hard to argue that the algorithm is correct: 

\begin{theorem}\label{thm alg correctness} Suppose Algorithm \ref{alg:find one} returns \textbf{success}. 
	 Then \( (\kbh,\wh) \) together with \( m \) is a weak exact abstraction.
	Suppose there are one or more weak exact abstractions only using  \(  \{P\sub h \su 1 (x, \ldots, y), \ldots, P\sub h \su k (x, \ldots, z)\} \), then Algorithm \ref{alg:find one} returns \textbf{success} with one such abstraction.  
	% a \( (\kbh,\wh) \) and \( m \) that is a weak exact abstraction where \( \kbh \) only uses  \(  \{P\sub h \su 1 (x, \ldots, y), \ldots, P\sub h \su k (x, \ldots, z)\} \). Then Algorithm \ref{alg:find one} returns \textbf{success}.
	Finally, Algorithm \ref{alg:find one} returns \textbf{failure} iff there is no weak exact abstraction.
	
	% \( (\kbh,\wh) \) and \( m \), where \( \kbh \) only uses the \(  \{P\sub h \su 1 (x, \ldots, y), \ldots, P\sub h \su k (x, \ldots, z)\} \), which is a weighted exact abstraction.
	
\end{theorem}

\begin{proof} Algorithm \ref{alg:find one} returns \textbf{success} only when the current guess \( (\kbh,\wh,m) \)  satisfies \eqref{eq cond weighted exact}. So soundness follows. Since we do exhaustive search, completeness follows. 
	Analogously, Algorithm \ref{alg:find one} returns \textbf{failure} iff none of the guesses from the exhaustive search satisfy \eqref{eq cond weighted exact}. 
\end{proof}

\subsection{Effective Testability} % (fold)
\label{sub:effective_testability}

% subsection effective_testability (end)

In the above algorithm,   the test \eqref{eq cond weighted exact} is challenging, because we would need to compute the probabilities for every formula. Even though we are assuming  a finite vocabulary, we would still  need to consider the set of all  well-defined formulas over Boolean connectives of arbitrary but finite length, which is  very large.  
Therefore, one might wonder if this test could be made easier, the most natural case being that of testing the probabilities of literals, which can be enumerated easily. Indeed, for a language with  \( k \)  predicates of arity \( w \) and a domain of size \( n, \) there will be at most \( k \times n^w \) atoms, and so at most \( 2\times k \times n^w  \) literals. We will now show that such a  result is possible, but we will need  \emph{separable} refinement mappings.

\begin{definition} Suppose \( \kbh \) and \( \kbl \) are theories. We say a refinement mapping \( m \) from \( \kbh \) to \( \kbl \) is \emph{separable} iff for every \( \alpha \land \beta \in \lang(\kbh) \) such that \( \alpha \) and \( \beta \) do not share (high-level) atoms, then it also the case \( m(\alpha) \in \lang(\kbl) \) and \( m(\beta) \in \lang(\kbl) \) do not share (low-level) atoms. 	
\end{definition}

We do not expect this stipulation to be problematic in the least, and it seems entirely natural. 
For separable refinement mappings, we obtain the following property: 

\begin{theorem}\label{defn weighted exact lits} Suppose \( (\kbh,\wh) \) and \( (\kbl,\wl) \) are two theories and \( m \) is a separable refinement mapping from \( \kbh \) to \( \kbl. \) Suppose 
		for all literals \( d\in \lits(\kbh) \), \( \Pr(d, \kbh, \wh) = \Pr(m(d), \kbl, \wl) \). Then for all \( \phi\in \lang(\kbh) \),  \( \pr(\phi,\kbh,\wh) =  \pr(m(\phi), \kbl,\wl). \) 
	
\end{theorem}

\begin{proof} 
Proof by induction on the length of \( \phi. \) The base case is immediate by definition, so assume for all \( \phi \in \lang(\kbh) \) of length \( k, \) \( \pr(\phi,\kbh,\wh) = \pr(m(\phi),\kbl,\wl) \). Consider \( \phi\in \lang(\kbh) \) of length \( k+1, \) and let \( d \) be any atom mentioned in \( \phi. \) Observe \( \phi \equiv (\phi\land d) \lor (\phi \land 
\neg d)   \). Let \( \phi\sub i \) be \( \phi \) but with every occurrence of \( d \) replaced by \( i \in \set{0,1} \) (denoting \emph{false} and \emph{true}) and every occurrence of \( \neg d \) replaced by \( 1-i. \) Observe that \( \phi \land d \equiv \phi\sub 1 \land d  \), that is, \( \phi \) is simplified by setting all occurrences of \( d \) with 1 (\emph{true}) and by setting all occurrences of \( \neg d \) with 0 (\emph{false}). Analogously, \( \phi \land \neg d \equiv \phi \sub 0 \land \neg d. \) 
Because \( \phi\sub i \) eliminates at least one  literal from \( \phi \), its length is \( \leq k \) and by hypothesis, \( \pr(\phi\sub i, \kbh, \wh) = \pr(m(\phi\sub i), \kbl, \wl) \). Moreover, since \( d \) is not mentioned in \( \phi\sub i \), it follows that \( \pr(\phi\sub i \land d, \kbh,\wh) = \pr(\phi\sub i, \kbh, \wh) \times \Pr(d, \kbh, \wh) \), and analogously for \( \pr(\phi\sub i \land \neg d, \kbh, \wh) \), \( \Pr(m(\phi\sub i) \land m(d), \kbl,\wl) \), and so on. 
	
	From Theorem \ref{thm wmc properties}, we know that \( \pr(\alpha \lor \beta) = \pr(\alpha)  + \pr(\beta) - \pr(\alpha\land\beta).\) Let us now apply this to \( 
	\phi \). So, \( \pr(\phi, \kbh, \wh) \) \begin{itemize}
		\item[$=$] $ \pr((\phi\sub 1 \land d) \lor (\phi\sub 0\land \neg d), \kbh, \wh)$
		\item[$=$]  \( \pr(\phi\sub 1\land d, \kbh, \wh) + \pr(\phi\sub 0\land \neg d,\kbh,\wh) - \pr((\phi\sub 1 \land d) \land (\phi\sub 0\land \neg d), \kbh, \wh) \)
		\item[\( = \)]  \( \pr(\phi\sub 1\land d, \kbh, \wh) + \pr(\phi\sub 0\land \neg d,\kbh,\wh)  \) because \(  \pr((\phi\sub 1 \land d) \land (\phi\sub 0\land \neg d), \kbh, \wh) = 0 \) owing to \( (\phi\sub 1 \land d) \land (\phi\sub 0\land \neg d) \) being inconsistent 
		\item[\( = \)]  \( \pr(\phi\sub 1, \kbh, \wh) \times \pr(d,\kbh,\wh)  + \pr(\phi\sub 0, \kbh, \wh) \times \pr(\neg d,\kbh,\wh) \) because \( d \) or its negation is not mentioned in \( \phi\sub i \)
		\item[\( = \)]  \( \pr(m(\phi\sub 1), \kbl, \wl) \times \pr(m(d),\kbl,\wl)  + \pr(m(\phi\sub 0), \kbl, \wl) \times \pr(m(\neg d),\kbl,\wl) \) by induction hypothesis 
		\item[\( = \)]  \( \pr(m(\phi\sub 1) \land m(d), \kbl, \wl)    + \pr(m(\phi\sub 0) \times m(\neg d), \kbl, \wl)  \)  owing to \( m \) being separable (that is, \( m(\phi\sub i) \) and \( m(d) \) do not share atoms)
		\item[\( = \)]   \( \pr(m(\phi\sub 1\land d), \kbl, \wl) + \pr(m(\phi\sub 0\land \neg d),\kbl,\wl) - \pr(m(\phi\sub 1 \land d) \land m(\phi\sub 0\land \neg d), \kbl, \wl) \) where \( \pr(m(\phi\sub 1 \land d) \land m(\phi\sub 0\land \neg d), \kbl, \wl) = 0 \) owing to its inconsistency (so is vacuously added) 
		\item[\( = \)]  \( \pr(m(\phi), \kbl,\wl) \). 
	\end{itemize}  
	
	% However, observe that \( \phi\sub 0\land\phi\sub 1 = (\phi\land d=1) \land (\phi \land d=0)  \), which is inconsistent, so has probability 0. Thus, \( \pr(\phi,\kbh,\wh) = \pr(\phi\sub 0, \kbh, \wh) + \pr(\phi\sub 1, \kbh, \wh). \) Analogously, it is not hard to show that \( \pr(m(\phi\sub 0) \lor m(\phi\sub 1), \kbl, \wl) \), that is, \( \pr(m(\phi), \kbl, \wl) = \pr(m(\phi\sub 0), \kbl, \wl) + \pr(m(\phi\sub 1), \kbl, \wl) \), and
	%
	Therefore, \( \pr(\phi,\kbh,\wh) = \pr(m(\phi), \kbl,\wl) \). \end{proof}

% One would assume that
% whatever \( \kbh \) may be

Thus, when restricting to separable mappings, we can replace \eqref{eq cond weighted exact} in Algorithm \ref{alg:find one} with: \begin{equation}\tag{\( \star\star \)}
	\text{for all literals \( d\in \lits(\kbh) \), \( \Pr(d, \kbh, \wh) = \Pr(m(d), \kbl, \wl) \)}
\end{equation}
In line 4, moreover, we would only be guessing separable mappings. 
It is then not hard to show that correctness still holds for the modified Algorithm \ref{alg:find one}, provided the existence of abstractions is  stipulated as being limited to separable mappings.\footnote{Observe that although our language is relational, most of the results essentially resort to ground theories.
This is not uncommon in the literature on statistical relational learning (e.g., \cite{inference-in-probabilistic-logic-programs}), and as argued earlier, exploiting the relational structure for computational purposes   \cite{lifted-inference-and-learning-in-statistical} is orthogonal to the main thrust of this work. Nonetheless, if we were to make further assumptions about the  theory, such as stipulating  that all instances of a predicate occur with the same probability \cite{symmetric-weighted-first-order-model}, we could perhaps simplify \( (\star\star) \) further. It might suffice, for example, to simply check the probabilities of any arbitrary instance of the predicate to ascertain abstractions.}

% section deriving_abstractions (end)

\subsection{Effective Testability: Beyond Weak Abstractions} % (fold)
\label{sec:beyond_weak_abstractions}

The constraints for weighted exact abstractions  are clearly more involved, which, at first glance, would involve enumerating models and checking for isomorphic structures. Although this may be possible via techniques like knowledge compilation \cite{a-knowledge-compilation-map}, one might wonder if simpler tests, like the ones asserted in \( (\star\star) \) above, could also be obtained for weighted exact abstractions. 
%
% including tests for properties asserted in, for example, Definition \ref{defn sound abs}. These latter tests involve enumerating models and checking that there are isomorphic structures. Although this may be possible via techniques like knowledge compilation \cite{a-knowledge-compilation-map}, one might wonder if simpler tests, like the ones asserted in \( (\star\star) \) above, could be obtained.
%
What is really needed in addition to \( (\star\star) \), of course, is a way to establish that \( \kbh \) is a sound and complete abstraction of \( \kbl \) relative to a mapping \( m. \) What we want to avoid, ideally, is a strategy for establishing the latter via tests involving the set of all  well-defined formulas like in \( (\star) \). We now show that this is indeed possible. 

Below, given a theory \( \kb, \) we sometimes think of it as a set of formulas \( \set{\phi\sub 1, \ldots, \phi\sub k} \), with the understanding that the theory is equivalent to \(  \phi\sub 1 \land \ldots \land \phi\sub k. \) We write \( \phi\sub i\in \kb \) to mean that \( \phi\sub i \) is one of the formulas of that set \( \kb. \) We will also assume separable mappings for establishing our results:

\begin{theorem} Suppose \( \kbh \) and \( \kbl \) are   logical theories, and \(  m \) is a separable refinement mapping from \( \kbh \) to \( \kbl \). Suppose for all  \( \phi \in \kbh, \) \( \kbl \models m(\phi) \). 
	% : \begin{enumerate}
% 	\item for all  \( \phi \in \kbh, \) \( \kbl \models m(\phi) \); and
% 	\item for all atoms \( p \in \lang(\kbh) \), if  \( \pr(m(p),\kbl,\wh) \gt 0 \) then \( \pr(p,\kbh,\wh) \gt 0 \).
%
% 	% ; and (b) if \( \pr(m(p),\kbl,\wl) = 1 \) then \( \pr(p,\kbh,\wh) = 1. \)
% %
% % 	\item for all atoms \( p \in \lang(\kbh) \), (a) if \( \pr(p,\kbh,\wh) \gt 0 \) then \( \pr(m(p),\kbl,\wh) \gt 0 \); and (b) if \( \pr(m(p),\kbl,\wl) = 1 \) then \( \pr(p,\kbh,\wh) = 1. \)
% \end{enumerate}
Then \( \kbh \) is a sound abstraction of \( \kbl \) relative to \( m. \)
	
\end{theorem}

\begin{proof}  Suppose \( \ml \) is any model of \( \kbl. \) Now, by assumption, for every \( \phi\in \kbh, \) \( \kbl \models m(\phi) \), that is, \( \kbl \models m(\kbh) \). Because \( \ml\in \mods(\kbl), \) \( \ml \models m(\kbh). \)

% Let \( p\sub 1, p\sub 2, \ldots, p\sub k \)	 be all the literals of \( \lang(\kbh) \) such that \( \ml \models  \)
Given all the atoms \( p\sub 1, \ldots, p\sub k  \) of $\lang(\kbh)$, let \( d\sub 1, \ldots, d\sub k \) be the literals (say, \( d\sub 1 = p\sub 1 \), \( d\sub 2 = \neg p\sub 2 \), and so on) such that \( \ml \models m(d\sub 1) \land \ldots \land m(d\sub k). \) (By separability, note that for atoms \( p,q\in  \lang(\kbh)\), \( m(p) \) and \( m(q) \) do not share low-level atoms.) Let \( \mh \) be an interpretation of \( \lang(\kbh) \) such that \( \mh \models d\sub 1 \land \ldots \land d\sub k. \) 
% for every atom \( p \in \lang(\kbh) \), \( \mh[p] = 1 \) iff \( \ml \models m(p) \).
	By construction then \( \mh \iso \ml. \) By Theorem \ref{thm iso}, for every \( \phi \in \lang(\kbh) \), \( \mh \models \phi \) iff \( \ml\models m(\phi) \). Since \( \ml \models m(\kbh) \), it follows that \( \mh \models \kbh \). Since such a high-level model can be constructed for any \( \ml\in \mods(\kbl) \), \( \kbh \) must be a sound abstraction of \( \kbl \) relative to \( m \). \end{proof}

We now turn to completeness:

\begin{theorem}\label{thm completeness proof test} Suppose \( \kbh \) and \( \kbl \) are logical theories, and \(  m \) is a separable refinement mapping from \( \kbh \) to \( \kbl \). Suppose for all literals \( d \in \lang(\kbh) \), if \( d\land \kbh \) is satisfiable then so is \( m(d) \land \kbl \). Then \( \kbh \) is a complete abstraction of \( \kbl \) relative to \( m. \)
% 	: \begin{enumerate}
% 	\item for all  \( \phi \in \kbh, \) \( \kbl \models m(\phi) \) (that is, for all sentences in \( \kbh \), which \( \kbh \) clearly entails, it is also the case that \( \kbl \) entails their mappings); and
% 	\item for all atoms \( p \in \lang(\kbh) \), if \( \pr(p,\kbh,\wh) \gt 0 \) then \( \pr(m(p),\kbl,\wh) \gt 0 \).
%
% 	% ; and (b) if \( \pr(m(p),\kbl,\wl) = 1 \) then \( \pr(p,\kbh,\wh) = 1. \)
% %
% % 	\item for all atoms \( p \in \lang(\kbh) \), (a) if \( \pr(p,\kbh,\wh) \gt 0 \) then \( \pr(m(p),\kbl,\wh) \gt 0 \); and (b) if \( \pr(m(p),\kbl,\wl) = 1 \) then \( \pr(p,\kbh,\wh) = 1. \)
% \end{enumerate}

\end{theorem}
% \begin{proof} Suppose \( \mh \) is any model of \( \kbh. \) Consider the formula \(
% 	\mh\su \downarrow \), which is the conjunction of literals that are true at \( \mh \). Because this is the unweighted setting, \( \pr(\mh\su\downarrow,\kbh,\wh) \gt 0 \). In particular, suppose \( \mh\su\downarrow = l\sub 1 \land \ldots \land l\sub k \). Clearly \( \pr(l\sub i, \kbh, \wh) \gt 0 \) because otherwise the conjunction \( \mh\su\downarrow \) would have had zero probability. By assumption, \( \pr(m(l\sub i),\kbl,\wl) \gt 0 \) for every \( i \). By separability, it follows that \( m(l
% 	\sub i) \) and \( m(l\sub j) \) are disjoint events for every \( i
% 	\neq j \), and so it follows that \( \pr(\bigwedge m(l\sub i), \kbl, \wl) \gt 0 \). Then it must be the case that there is a \( \ml \in \mods(\kbl) \) such that \( \ml \models \bigwedge m(l\sub i) \). Thus for every atom \( p \in \lang(\kbh) \), \( \mh \models p \) iff \( \ml\models m(p) \), so \( \mh \iso \ml. \) \end{proof}

\begin{proof} Suppose \( \mh \) is any model of \( \kbh. \) Consider the formula \(
	\mh\su \downarrow \), which is the conjunction of literals that are true at \( \mh \).
	% Because this is the unweighted setting, \( \pr(\mh\su\downarrow,\kbh,\wh) \gt 0 \).
	In particular, suppose \( \mh\su\downarrow = d\sub 1 \land \ldots \land d\sub k \). Clearly, \( d\sub i \land \kbh \) is satisfiable, and so by assumption, \( m(d\sub i) \land \kbl \) is satisfiable: let \( \ml \su i \) be such a model where \( m(d\sub i) \land \kbl \) is true. In particular, let \( l\su i\sub 1 \land \ldots \land l\su i\sub {u\sub i} \) be a conjunction of literals true at \( \ml\su i \) such that these literals mention all the atoms in \( m(d\sub i) \) and only them. Put differently, \( l\su i\sub 1 \land \ldots \land l\su i\sub {u\sub i} \models m(d\sub i) \), and we can also see \(  l\su i\sub 1 \land \ldots \land l\su i\sub {u\sub i} \) as a partial interpretation for \( \kbl. \) When we consider such partial interpretations for \( i\neq j, \) by separability it follows that \( m(d
	\sub i) \) and \( m(d\sub j) \) do not share atoms, so \( L\su i =  l\su i\sub 1 \land \ldots \land l\su i\sub {u\sub i} \) and \( L\su j =  l\su j\sub 1 \land \ldots \land l\su j\sub {u\sub j} \) do not share atoms, and  thus are consistent with each other. In other words, we now have the partial interpretation \( L\su 1 \land \ldots \land L\su k \) of \( \kbl \) such that for each \( i\): (a) \( L\su i \models m(d\sub i) \), (b) \( L\su i \) mentions all and only the atoms in \( m(d\sub i) \). Let \( \ml \) be an  interpretation of \( \kbl \) where \( L\su 1 \land \ldots \land L\su k \) holds. Then for every atom \( p \in \lang(\kbh) \), \( \mh \models p \) iff \( \ml\models m(p) \), so \( \mh \iso \ml. \) \end{proof}

	% \( \pr(l\sub i, \kbh, \wh) \gt 0 \) because otherwise the conjunction \( \mh\su\downarrow \) would have had zero probability. By assumption, \( \pr(m(l\sub i),\kbl,\wl) \gt 0 \) for every \( i \). By separability, it follows that \( m(l
	% \sub i) \) and \( m(l\sub j) \) are disjoint events for every \( i
	% \neq j \), and so it follows that \( \pr(\bigwedge m(l\sub i), \kbl, \wl) \gt 0 \). Then it must be the case that there is a \( \ml \in \mods(\kbl) \) such that \( \ml \models \bigwedge m(l\sub i) \). Thus for every atom \( p \in \lang(\kbh) \), \( \mh \models p \) iff \( \ml\models m(p) \), so \( \mh \iso \ml. \) \end{proof}

Putting it all together, in Algorithm \ref{alg:find one}, to test whether a guess \( (\kbh,\wh,m) \) is a weighted exact abstraction, we would need three checks: \begin{enumerate}
	\item for all \( \phi\in\kbh, \kbl\models m(\phi) \); 
	\item for all literals \( d\in \lang(\kbh), \) if \( d\land \kbh \) is satisfiable then so is \( m(d)\land \kbl \); and 
	\item \( (\star\star) \). 
\end{enumerate}

So of course the test involves more computations than for weighted exact abstractions, but as we discuss above, we additionally obtain logical alignment should that be desired. 

% section beyond_weak_abstractions (end)

\subsection{Constraining search} % (fold)
\label{sub:discussion}

% subsection discussion (end)

Algorithm \ref{alg:find one} is discussed at a  general level, which is deliberate, but that also means that no commitment has been made yet on how to define and constrain the search space.\footnote{One could imagine that if the high-level predicates are not provided as input for Algorithm \ref{alg:find one}, we might then parameterize the algorithm by providing a vocabulary bound \( k \) and an arity  bound \( z \), and attempt to guess a high-level theory from  \( k \)   predicates of maximum arity \( z \). If that fails, the bound could be incremented.} 
  % Note that with \( n \) nullary predicate symbols, there are \( 2\su n \) possible constructions of \( \kbh, \) and thus, one might wonder whether reasonable strategies exist to constrain the search space.
The case of formula replacements was discussed in Section \ref{sub:syntactic_substitutions}.  Below, we discuss two other strategies. We reiterate that to make the process of deriving abstractions effective some combination of such strategies along with a suitably restricted fragment is likely needed. \smallskip 

%  that could be used in
%
% conjunction with formula replacements, and perhaps even
% each other.

% but in the absence of that, many variants are possible for keeping the search still general while restricting the space of solutions.

\textbf{Partial knowledge:} The user might be in a position to suggest (and constrain) the space of possible mappings and the syntactical form of the high-level theory. 
%In that case, we motivate an approach analogous to inductive logic programming below \cite{inductive-logic-programming:-theory}. 
For example, for some unknown \( 
\kbh \) over the predicates \(  \{P\sub h \su 1 (x, \ldots, y), \ldots, P\sub h \su k (x, \ldots, z)\} \), suppose the user decides to use the same domain as \( \kbl \). Suppose she provides partial information about the refinement mapping: let \( m' \) be a mapping from atoms \( p \in L \) to \( \theta\sub {p} \), where \( \theta\sub {p} \in \lang(\kbl) \) and \( L \subseteq \lits(\kbh) \). 
%Then we are clearly after some \( m \) that extends \( m'. \) 
%In set notation, \( m' \subseteq m. \) 
Let us further assume the user provides partial knowledge of the sentences in \( \kbh, \) say \( \kbh'. \) Then line 4 of Algorithm \ref{alg:find one} would guess  functions \( m \) that extends \( m' \)  
(that is, possible completions of \( m \)), and line 3 would guess \( \kbh \) such that \( \kbh' \subseteq \kbh. \)  It would then follow that with exhaustive search, correctness would still be shown to hold, provided  the existence of abstractions is stipulated as being limited to extensions of the partial knowledge  (that is, the partial knowledge is assumed to be correct in the simplest case). If the user were to provide examples \( \set{e\sub 1, \ldots, e\sub k} \) instead of  a sub-theory \( \kbh' \), we might take an approach akin to inductive logic programming \cite{inductive-logic-programming:-theory}. That is, we first define a syntactic bias, that is, a hypothesis space \( 
\cal H \), and find a \( \kbh \in \cal H \) satisfying a semantic bias (say, \( \kbh \models e\sub 1 \land \ldots \land e\sub k. \))  In this case, if Algorithm \ref{alg:find one} were to return \textbf{success}, soundness is immediate, but further investigations are needed to show that an appropriate \( 
\cal H \) also promises completeness. \smallskip

\textbf{Decomposability:} Rather than searching for abstractions by treating \( \kbl  \)  as a monolithic entity, a pragmatic alternative is possible when the low-level theory is \emph{decomposable}: that is, it is a set of sentences that are logically independent of each other. Then, we can identify \emph{local} abstractions and compose those to obtain a \emph{global} solution. Consider the case where \( 
\kbl = \phi\sub 1 \land \ldots \land \phi\sub k \), where \( \phi\sub i \) does not share  atoms with \( \phi\sub j  \) for all \( i \neq j. \) Such decompositions may appear naturally when a joint distribution is characterized over a set of disjoint Markov networks \cite{inducing-features-of-random-fields,markov-logic-networks,probabilistic-inference-in-hybrid-domains}, or may be obtained by knowledge compilation \cite{a-knowledge-compilation-map} from a more involved theory. Recent advances in tractable learning \cite{gens2013learning}, which have their roots in knowledge compilation, also identify clusters of random variables that are independent of each other. It  follows that \( 
\pr(\phi\sub i \land \phi \sub j) = \pr(\phi\sub i) \times \pr(\phi \sub j) \) and \( \pr(\phi\sub i \lor \phi\sub j) = \pr(\phi\sub i) + \pr(\phi\sub j) - \pr(\phi\sub i) \times \pr(\phi\sub j) \). Basically then we can identify \( \kbh = \psi\sub 1 \land \ldots \land \psi\sub k \), where  \( \psi\sub i \) abstracts \( \phi\sub i, \) and shares the structural restriction that \( \psi\sub i \) and \( \psi\sub j \) do not share  atoms for all \( i\neq j. \) The abstraction search would be limited locally to \( \phi\sub i \), that is, for example, wrt each low-level Markov network. 

Fortunately, we are able to show that this intuitive idea is correctness preserving. We prove the case of complete abstractions, and the other cases are analogous. 

\begin{theorem} Suppose \( \psi\sub i \) is a complete abstraction of \( \phi\sub i \) relative to \( m\sub i. \) Suppose \( \phi\sub i \) and \( \phi\sub j \) do not share  atoms for all \( i\neq j, \) and analogously, \( \psi\sub i \) and \( \psi\sub j \) do not share  atoms for all \( i\neq j. \) 
	Then \( \kbh = \psi\sub 1 \land \ldots \land \psi\sub k \) is a complete abstraction of \( \kbl = \phi \sub 1 \land \ldots \land \phi\sub k \) relative to \( m = m\sub 1 \land \ldots \land m\sub k \) (that is, the composite mapping obtained by extending \( m\sub 1 \) to include the vocabulary and mapping of \( m\sub 2 \), which is then extended for \( m\sub 3 \), and so on.) 
	
\end{theorem}

\begin{proof} Suppose \( 
	\mh \) is any model of \( \kbh. \) Consider the formula \( \mh\su \downarrow \), which is the conjunction of literals that are true at \( \mh. \) Consider that \( \mh\su\downarrow \) can be written as \( h\sub 1 \su \downarrow \land \ldots \land h\sub k\su \downarrow \), where: \begin{itemize}
		\item \( h\sub i \) is an interpretation for \( \lang(\psi\sub i) \); 
		\item  following our notation, \( h\sub i\su \downarrow \)  is a conjunction of literals; and so
		\item \( h\sub i \su \downarrow \) only mentions the atoms from \( \lang(\psi\sub i) \).
	\end{itemize}  By construction, since \( \mh \) is a model of \( \kbh, \) 
	 \( h\sub i \) must be a model of \( \psi\sub i.  \) By assumption, there is a model \( l\sub i \) of \( \phi \sub i \) such that \( h\sub i \) is isomorphic to \( l\sub i \) relative to \( m\sub i. \) Then let \( \ml \) be the model corresponding to the formula \( \ml\su\downarrow = l\sub 1 \su\downarrow \land \ldots \land l\sub k\su \downarrow \). By construction, \( \ml \) must be a model of \( \kbl \), and moreover, from the isomorphism that holds for \( h\sub i \) and \( l\sub i \) relative to \( m\sub i \), \( \mh \iso \ml. \) \end{proof}

\section{Related Work and Discussion} %
\label{sec:related_work}

Abstraction is a major topic in knowledge representation  \cite{giunchiglia1992theory,erol1996complexity,saitta2013abstraction,banihashemi2017abstraction}. The idea of establishing mappings between models to yield a semantic theory for abstraction owes its origin to works such as \cite{milner1989communication}. But formal treatments have been mostly  restricted to categorical and non-probabilistic domains. Nonetheless, there have been various developments in different communities, and we discuss the lineage in more detail below. \smallskip

\textbf{Knowledge representation and automated planning:}  
Our framework here is inspired by, and  indeed builds on the proposal in \cite{banihashemi2017abstraction}, where isomorphism  as well as sound and complete abstractions are investigated for (non-probabilistic) situation calculus agent programs.  Here, we sought to extend those ideas to establish probabilistically interesting properties  for probabilistic models, including probabilistic relational models (PRMs). For this, we    motivated the notion of unweighted abstractions, which roughly corresponds (at the level of satisfaction and entailment)  to the categorical setting \cite{banihashemi2017abstraction}. But  by piggybacking on these properties, weighted abstractions and evidence incorporation were motivated and formulated. We then identified the relationship of that framework to a  purely stochastic one (weak exact abstractions), and further studied the automatic derivation of abstractions. Our observations about decomposability, among other things, need not be limited to stochastic models and 
may also be applicable in the categorical setting.

% , and so may be of interest to approaches such as \cite{banihashemi2017abstraction} that we built on.

We refer interested readers to  the references in \cite{banihashemi2017abstraction} for a comprehensive discussion on the use of abstraction in knowledge representation and automated planning (e.g., hierarchical planning). A particularly interesting direction in this landscape is the work of \cite{giunchiglia1992theory}, where operations on abstractions are studied at the level of the logical theories. We suspect these results can be generalized (and adapted) to our framework with some effort. \smallskip 

%, which we are keen to investigate in the future. 

 %

\textbf{Program synthesis and verification:} 
In the area of program verification, static analysis and abstraction interpretations are commonplace to test the correctness of programs and probabilistic programs. See \cite{mciver2005abstraction} for a book length treatment, for example. 
A number of additional concerns present themselves in a programmatic setting, including the branching in the presence of stochastic primitives, and stochastic  transitions between program states. The motivation then is to deduce sound abstractions for verifying correctness (e.g., termination) properties.  
While some of these works do not consider  abstractions themselves to be probabilistic, 
the  developments are  related to our goals. Thus, it would be interesting to study how ideas and techniques from the program analysis literature can be carry over to our framework, and vice versa. 
For example, \cite{Zhang:2017:CLP:3088525.3088563} consider  statistical properties of program behavior to advise abstractions, \cite{sharma2013verification} relate  verification to the learnability of concepts, \cite{holtzen2017probabilistic} study  abstract predicates for loop-free probabilistic programs, and \cite{monniaux2001abstract} defines abstract representations for probabilistic program path analysis. On the semantical front,  \cite{cousot2012probabilistic} present a detailed and careful analysis for reasoning about (probabilistic) nondeterminism in programs, but as argued in \cite{holtzen2017probabilistic}, they do consider the abstractions themselves to be probabilistic structures. This is precisely the focus of \cite{holtzen2017probabilistic}, where they want to abstract loop-free probabilistic programs as possibly simpler and smaller probabilistic programs. These latter programs may then be amenable to automated verification, for example. Roughly, the idea is to transform a concrete program (in our terms: low-level) to an abstract program (in our terms: high-level). 
So this work is closest in spirit to our thrust. They mainly motivate a notion of sound probabilistic over-approximation, with the intent of capturing  a distribution over feasible states. In a sense, this is akin  to Definition \ref{defn weighted sound complete abs}, but without any stipulation of logical alignment, so it is \emph{weak} (as  in Section \ref{sec:weak_abstractions}). In follow up work, \cite{holtzen2018sound} introduce the notion of distributional soundness, which asserts that the probability of  a high-level event is equal to the probability of the low-level event. Thus, in a sense, this is akin  to weak exact abstractions.  Finally, they also investigate a strategy for finding high-level  programs that roughly corresponds to replacing atomic formulas in the low-level program with an appropriate high-level random variable \cite[Algorithm 1]{holtzen2018sound}. That procedure is similar in spirit to our generic search procedure, in the sense of requiring exponential search in the worst case, but considerably easier given the predefined structural syntax of the sought after abstraction. They discuss an implementation that also  leverages many state-of-the-art techniques  from the program verification community, such as counterexample-guided refinement \cite{clarke2000counterexample}. 
Put differently, the setting of loop-free probabilistic programs built from Bernoulli or other univariate distributions is often much more constrained  than a first-order language in that the grammar only allows conjunctions of positive statements (as in most loop-free sequential programs), and programs allow the use of the predicate transformer semantics \cite{dijkstra2012predicate} to abstract atomic assertions. 
Our setting is more general, in that: (a) high-level abstractions may map onto arbitrarily complex well-defined low-level formulas, (b) along with weak abstractions, we also motivate theory alignment (such as Definition \ref{defn weighted exact}), and (c) the treatment of evidence at the first-order/logical level allows for a richer perspective.  Nonetheless, restricted settings such as the one investigated for probabilistic programs may represent cases that offer reasonable  expressiveness/tractability tradeoffs. In that regard, refining schemas such as Algorithm \ref{alg:find one} to also  leverage state-of-the-art techniques  from the program verification community may identify other interesting fragments.

Finally, in the context of inductive logic programming and meta-interpretative learning, there is a long history of predicate invention and learning abstract programs  \cite{cropper2016learning} and their affect on human comprehensibility \cite{predicate-invention-comprehensibility}. Although the semantic constraints for abstracting is somewhat different, it is possible that our setup regarding partial knowledge in Section \ref{sec:deriving_abstractions} could be realized using such methods. \smallskip 

% For example, given a low-level theory \( \kbl \) and a target high-level first-order vocabulary, can existing program verification techniques be used to obtain candidate high-level theories?

\textbf{Probabilistic logical modeling:} 
In the PRM literature, a few recent developments seem to be  close in spirit to abstraction. In \cite{exploiting-shared-correlations-in-probabilistic}, the question of making inference more efficient in classes of probabilistic databases (PDBs) that share certain structural properties  is investigated. Roughly, what they are after is a possibly ``compressed" PDB that answers queries exactly as would the original PDB in the manner that inference computations are not repeated for the shared features. Of course, our framework differs in that the high-level and low-level theory do not need to have any structural similarities. Moreover, if they do share structural similarities,  at this point, we disregard the issue of how probabilistic computations can be made efficient, as this is somewhat orthogonal to the main thrust of the paper. We suspect, under some conditions, it might be possible to show that classes of PDBs with shared features may correspond to abstractions, but conversely, reiterating the point above, simply because \( \kbh \) and \( \kbl \) are abstractions need not imply that they share structural features. Along these lines, a recent thrust in PRMs is the question of how to make inference more efficient by exploiting the relational vocabulary, referred to as ``lifted reasoning''   \cite{lifted-inference-and-learning-in-statistical, bisimulation-based-approximate-lifted-inference}. This is justifiably sometimes referred to as a type of ``abstraction'' \cite{ludtke2018state}. 
There seem to be two implications  for our work. The first is computational: verifying that \( \kbh \) is an abstraction of \( \kbl \) could be made more efficient by appealing to lifted reasoning. (This is, as argued elsewhere, somewhat orthogonal to the main thrust of the paper.) The second revisits our observations about compressed PDBs  \cite{exploiting-shared-correlations-in-probabilistic}. One could, for example, see a non-ground PRM as the high-level abstraction of the low-level ground PRM, in that none of the domain constants are explicitly mentioned in the former. So, in that sense, a non-ground PRM would turn out to be an abstraction of a ground PRM, but simply because \( \kbh \) and \( \kbl \) are abstractions need not imply that they share the same vocabulary. For the future, we hope to study the connections between these strands of work and abstraction in more detail, so as to attempt to  formalize these  intuitions. 

On the topic of reasoning, our framework has some overlap with the principles of \emph{representation-independent} probabilistic inference  \cite{halpern2004representation}. Here,
 one is usually interested in the computed conditional queries not being different if the knowledge base is represented differently,  (say) using an abstract logical language. 
  So,   \cite{halpern2004representation}  motivate a notion of correctness  where if a query \( \phi \) follows a knowledge base \( \kb, \) it should also be the case that \( f(\phi) \) follows from \( f(\kb) \) where \( f \) is a mapping from one representation to another. Then, a \emph{robust} inference procedure is one that respects semantically justifiable properties such as those in Theorem \ref{thm wmc properties} for reasonable mappings.  
% Then, a mapping is considered to be \emph{faithful} iff for all formulas \( \phi \), it is entailed by \( \kb \) iff the mapping of \( \phi \) is entailed by the mapping of \( \kb. \) 
While there is some similarity at first glance, there is a crucial conceptual difference: as already argued in \cite{halpern2004representation}, 
unlike the work in representation-independent inference, 
the two representations  are not required to be equivalent in an exercise on abstraction, because, by definition, an abstraction ignores irrelevant facts. The technical thrust is also different in our work, such as the identification of weighted exact vs weak exact abstractions, the handling of evidence, and our analysis on generating abstractions. 
%It might be worthwhile, however, to study the proposals more closely, especially in the context of the point we make below on 
Following the work in \cite{halpern2004representation}, a broader treatment is given in \cite{jaeger1996representation}, where the notion of representation independence is studied for  non-classical consequence more generally, of which probability measures is a special case. Very similar in spirit to our own work as well as the categorical setting \cite{banihashemi2017abstraction} that we build on, the purely logical question of when two sentences represent the same information is considered first, prompting a definition that is virtually identical to unweighted sound and complete abstractions. As argued above for  the case of \cite{halpern2004representation}, there are numerous differences in terms of technical thrust, however. In our work, for example: (a)  abstractions were analyzed at the level of soundness and completeness;  (b) weighted abstractions were derived by piggybacking on constraints noted in the unweighted setting; (c) we investigated the difference between probabilistic alignment in the presence and absence of logical alignment; (d) we considered the incorporation of evidence; and (e) we identified properties for the verification and generation of abstraction. \smallskip 

\textbf{Statistics:} When establishing the alignment between the high-level and low-level theory, we  looked solely at the marginal probabilities  between two discrete probability distributions. Summarizing distributions is a standard problem in statistics, and the use of means and moments is common \cite{machine-learning:-a-probabilistic-perspective}; it is an interesting question whether such constructs could be useful for defining and/or deriving abstractions (in our sense). 

A more standard case of statistical abstraction is when continuous distributions are cast as discrete events by appealing to the cumulative distribution function. We touched upon this in Section  \ref{sub:syntactic_substitutions}, and we note that such ideas have been used for  inference  in continuous domains via a generalized variant of weighted model counting~\cite{probabilistic-inference-in-hybrid-domains},  and for probabilistic program abstractions  \cite{holtzen2017probabilistic}.

% Moreover, continuous distributions can be  cast in terms of discrete events by appealing to the cumulative distribution function. This idea has been used for  inference  in continuous domains via a generalized variant of weighted model counting \cite{probabilistic-inference-in-hybrid-domains}. In an analogous manner,  probabilistic program abstractions can be constructed by appealing to Boolean abstractions \cite{holtzen2017probabilistic}.
%
% For example, we could capture (say) a uniform distribution for a random variable \( X \) on \( (0,10) \subset \real \) wrt evidence \( X\geq 8 \)
% as a Boolean random variable \( B \) that is true with probability \( 0.8, \) and false with probability \( 0.2. \)
%
% , as already discussed in Section \ref{sec:deriving_abstractions}.

% Investigations on enforcing invariance,
% where  inference  based on
% observations do not depend on the (unknown)  sample space
%
%
% It would worthwhile to look at such topics from the viewpoint of our formal framework,

In causal modeling, mapping macro and micro level events is  a long-standing  concern, which correspond in our terms to hig-level and low-level models. 
In recent work, for example,  \cite{rubenstein2017causal} study  consistency between  micro and macro-level random variables via structural equation models, and so are close in spirit to abstractions. \smallskip

Let us conclude this section with some  observations. 
Firstly, despite the tremendous amount of attention that abstraction has received, the framework presented here is  useful for a number of reasons: (a) it is downward compatible with the categorical (and first-order) setting \cite{banihashemi2017abstraction}; (b) it identifies probabilistic alignment together with a notion of logical alignment, the former  closely mirroring the analysis on probabilistic program abstractions  \cite{holtzen2017probabilistic,holtzen2018sound}; and (c) it seems to agree with the intuitions regarding representation-independence probabilistic inference   \cite{halpern2004representation}; and (d) it can  leverage the advances in weighted model counting, including lifted reasoning. Let us now reflect on a few critical points. 

It is interesting to note that although the level of generality of our framework allows us to easily draw comparisons to results from the categorical literature and  representation independence, which is useful from a  theoretical standpoint,  
it may mean that the formal results are  somewhat removed from the concerns of  high-level  modeling languages. For example, a major issue with PRMs is ensuring that the ground model is acyclic so that inference and sampling methods can be  designed easily. To understand how those issues carry over to abstraction algorithms, we need to understand how high-level PRMs can be designed that ensure that such properties are not lost on abstraction if present in the low-level PRM. Conversely, even if the low-level PRM does not enjoy effective inference properties, can a high-level PRM be obtained that is amenable to those properties? In this work, since our focus was on understanding the semantical properties of abstractions, such concerns are orthogonal but they may impact our choice for guessing an appropriate high-level theory. 

To that end, existing empirical observations on abstraction-type frameworks are somewhat mixed. Early work on categorical abstractions  \cite{giunchiglia1992theory}, for example, noted that `` \( \ldots \) shows that there are situations where abstraction saves time but also situations where it results in less efficiency." In a similar vein, as discussed in \cite{holtzen2018sound}, factor graph abstractions for probabilistic programs do not always maintain structural decompositions.
% While the literature does not address these questions directly (outside of the special fragments and contexts discussed above), existing results on tangential concerns are somewhat mixed. 
On the one hand, if one takes a worst-case view of the inference problem, it immediately follows that given reasoning with (say) \( n \) atoms in \( \kbl \) versus \( m \ll n \) atoms in \( \kbh \), the latter seems preferable. For the case of  loop-free sequential probabilistic programs built from Bernoulli random variables and conditional statements, \cite{holtzen2018sound} also empirically show that abstractions can be effective. Moreover, as already mentioned, abstraction is a very successful strategy in the verification communities. Thus, we think there is a broader question of choosing the most reasonable abstraction, one that is indeed amenable to effective inference, perhaps more so than the low-level theory. 

% On the other hand, as discussed in \cite{holtzen2018sound}, factor graph abstractions for probabilistic programs do not always maintain structural decompositions. In  similar vein, early work on categorical abstractions  \cite{giunchiglia1992theory} noted that `` \( \ldots \) shows that there are situations where abstraction saves time but also situations where it results in less efficiency."
%   Thus,
  
  A parallel concern is about the choosing of a vocabulary that offers maximum comprehensibility \cite{predicate-invention-comprehensibility}. A minor observation to be made here is that if comprehensibility is all we care about (as opposed to being concerned about both comprehensibility and tractability), then it is not clear that one would need to fully abstract a theory. Indeed, we could introduce/invent definitional or semi-definitional predicates of the form \( P(\vec x) \supset \phi(\vec x) \) or \( \phi(\vec x) \supset P(\vec x) \), where \( \phi(\vec x) \) corresponds to low-level information, the idea being that the user is only exposed to instances of \( P \). Such a construction would invariably increase the size of the theory, but we would be taking steps toward  comprehensibility by designing  predicates to correspond closely to the user's vocabulary. In fact, the user could be exposed to a combination of high-level and low-level predicates  to provide explanations of a suitable granularity.

  Such concerns, of course, affect all frameworks on abstraction, and are not unique to our endeavor. A semantic theory of correctness such as the one considered in this paper, in that regard, would indeed be at the level of interpretations and not necessarily in the language of the user-specific domain knowledge, perhaps analogous to a theory on program correctness. We would not expect a mathematical analysis on correct programs to necessarily involve the syntactic constructs of the programming language represented as is, but rather in terms of  suitable mathematical objects that captures its computational underpinnings   (e.g., variable assignment, recursion). In this sense, a semantical theory on abstraction is  (in the non-technical sense) {\it abstract}.  Owing to our observations and  building on this theory,  we do hope that it becomes possible to define and inspect the choosing of effective abstractions in the future.

  % over a specific grammar to radically change depending on the syntactic constructs offered by a particular programming language. We may expect specifics of that language to make the verification of properties easier or harder (e.g., strongly typed vs higher-order capabilities), but fundamentally, the idea of what it means for a program to be correct is analyzed at the level of states, transitions and their properties, which is   (in the non-technical sense) {abstract}.

%
%
%
%
\section{Conclusions} % (fold)
\label{sec:conclusions}

% section conclusions (end)

In this work, 
we were motivated in the development of a formal framework for abstractions, based on  isomorphisms between models, where atoms in a high-level theory can be mapped to complex formulas at the low-level. 
From that, we developed a number of  accounts of abstraction, as well as the handling of  low-level evidence, all of which motivated some observations about how abstractions can be derived automatically. 

%The results presented tackles the topic in a general way, only assuming that the domain in question be represented 
%Our account was based on WMC, which also serves as an algorithmic regime  on how to compute and compare probabilities. 
%

Given the increasing interest in abstraction for  statistical  AI, we hope our framework will be helpful in developing probabilistic abstractions for increased clarity, modularity and tractability, and perhaps interpretability. \smallskip

\bibliographystyle{abbrv}
% %\bibliographystyle{ACM-Reference-Format}
% %\bibliographystyle{acmauthoryear}
%

% \clearpage
%
% \section*{APPENDIX: PROOFS}
%
% \printproofs

\end{document}